\documentclass[journal]{IEEEtran}
\usepackage[utf8]{inputenc} 
\usepackage[pdftex]{graphicx} 

\usepackage{epstopdf}

\usepackage{amssymb}
\usepackage{amsthm}

\usepackage{amsfonts}
\usepackage{amssymb}
\usepackage{indentfirst}
\usepackage{dsfont}
\usepackage{color,colortbl}
\usepackage{stmaryrd}
\usepackage{cancel}

\usepackage{dsfont}
\usepackage{mathrsfs}
\usepackage{tikz}
\newtheorem{theorem}{Theorem}[section]
\newtheorem{lem}[theorem]{Lemma}
\newtheorem{prop}[theorem]{Proposition}

\newtheorem{condition}[theorem]{Condition}

\newenvironment{example}[1][Example]{\begin{trivlist}
\item[\hskip \labelsep {\bfseries #1}]}{\end{trivlist}}
\newenvironment{rmq}[1][Remark]{\begin{trivlist}
\item[\hskip \labelsep {\bfseries #1}]}{\end{trivlist}}

\newcommand{\E}{\operatorname{\mathbb{E}}}

\newcommand{\Var}{\operatorname{Var}}
\newcommand{\Cov}{\operatorname{Cov}}

\newcommand{\argmin}{\operatorname{argmin}}
\newcommand{\W}{\mathcal{W}_2(\mathbb{R})}
\newcommand{\Tr}{Tr}

\def\bbP{\mathbb{P}}

\def\bbR{\mathbb{R}}

\def\ML{\hat{\theta}_{ML}}
\def\st{\sup\limits_{\theta \in \Theta}}
\def\mi{\max\limits_{i=1\cdots p}}
\def\yc{\hat{Y}_{\ML}}
\def\ycz{\hat{Y}_{\theta_0}}

\usepackage{cite}

\ifCLASSINFOpdf
\else
\fi
\usepackage[cmex10]{amsmath}
\usepackage{array}
\usepackage{dblfloatfix}

\usepackage{color,colortbl}
\usepackage{hyperref}

\begin{document}
%
\title{A Gaussian Process Regression Model\\ for Distribution Inputs}
%
%
%

\author{Fran\c{c}ois~Bachoc, Fabrice~Gamboa, Jean-Michel~Loubes and Nil~Venet
\thanks{The four authors are affiliated to the Institute of Mathematics of Toulouse, Universit\'e Paul Sabatier, Toulouse, France. NV is also affiliated to, and completely funded by CEA. E-mail: (francois.bachoc, jean-michel.loubes, fabrice.gamboa,nil.venet)@math.univ-toulouse.fr. An overview of the main results of this article was described in the conference article \cite{SFDS}. }
}

\maketitle

\begin{abstract}
Monge-Kantorovich distances, otherwise known as Wasserstein distances, have received a growing attention in statistics and machine learning as a powerful discrepancy measure for probability distributions. In this paper, we focus on forecasting a Gaussian process indexed by probability  distributions. For this, we provide a family of positive definite kernels built using transportation based distances. We provide a probabilistic understanding of these kernels and characterize the corresponding stochastic processes.  We prove that the Gaussian processes indexed by distributions corresponding to these kernels can be efficiently forecast, opening  new perspectives in Gaussian process modeling.
\end{abstract}


%
\IEEEpeerreviewmaketitle

\section{\textbf{Introduction}} \label{section:introduction}
%
%
%
%
\IEEEPARstart{O}{riginally} used in spatial statistics  (see for instance \cite{MR1127423} and references therein), Kriging has become very popular in many fields such as machine learning or computer experiment, as described in \cite{MR2514435}. It consists in predicting the value of a function at some point by a linear combination of observed values at different points. The unknown function is modeled as the realization of a random process, usually Gaussian, and the Kriging forecast can be seen as the posterior mean, leading to the optimal linear unbiased predictor of the random process. \vskip .1in
Gaussian process models rely on the definition of a covariance function that characterizes the correlations between values of the process at different observation points.
As the notion of similarity between data points is crucial, \textit{i.e.} close location inputs are likely to have similar target values,  covariance functions are the key ingredient in using Gaussian processes, since they define nearness or similarity.
In order to obtain a satisfying model one need to chose a covariance function (\textit{i.e.} a positive definite kernel) that respects the structure of the index space of the dataset.
Continuity of the covariance is a minimal assumption, as one may ask for additional properties such as stationarity or stationary increments with respect to a distance. These stronger assumptions allow to obtain a model where the correlations between data points depend on the distance between them.

First used in Support Vector (see for instance \cite{vapnik1997support}), positive definite kernels are nowadays used for a wide range of applications. There is a huge statistical literature dealing with the construction and  properties of  kernel functions over $\mathbb{R}^d$ for $d \geq 1$ (we refer for instance to \cite{scholkopf2002learning} or \cite{cristianini2000support} and references therein).  Yet the construction of kernels with adequate properties on more complex spaces is still a growing field of research (see for example \cite{cohenlifshits}, \cite{istas2011manifold}, \cite{feragen2015geodesic}). \vskip .1in

Within this framework, we tackle the problem of forecasting  a process  indexed by one-dimensional distributions. Our motivations come from a variety of applied problems: in the classical ecological inference problem (see \cite{flaxman2015supported}), outputs are not known for individual inputs but for groups, for which the distribution of a covariate is known. This situation happens for instance in political studies, when one wants to infer the correlation between a vote and variables such as age, gender or wealth level, from the distributions of these covariates in different states (see for example \cite{flaxman2015supported}). The problem of causal inference can also be considered in a distribution learning setting (see \cite{lopez2015towards}).

As \cite{muandet2017kernel} remarks, learning on distribution inputs offers two important advantages in the big data era. By bagging together individual inputs with similar outputs, one reduces the size of a dataset and anonymizes the data. Doing so results on learning on the distribution of the inputs in the bags.

Another application arises in numerical code experiments,  when  the prior knowledge of the input conditions may not be an exact value but rather a set of acceptable values that will be modeled using a prior distribution. Hence we observe output values for such probability distributions and want to forecast the process for other ones. A similar application of distribution inputs for numerical code experiments is given by non-negative functional inputs. We give a detailed example of this situation in Section \ref{section:IRSN}.  \vskip .1in

Several approaches already exist to deal with distribution inputs regression. An important class of methods relies on some notion of divergence between distributions (see \cite{poczos2012nonparametric,poczos2013distribution,barnabas2012nonparametric}). Other methods have been proposed, such as kernel mean embedding \cite{muandet2017kernel} and kernel ridge regression methods \cite{szabo2015two}. In this paper we focus on Gaussian process regression method.  \vskip .1in

The first issue when considering Gaussian process regression for distribution inputs is to define a covariance function, which will allow to compare the similarity between probability distributions. Several approaches can be considered here. The simplest method is to compare a set of parametric features built from the probability distributions, such as the mean or the higher moments. This approach is limited as the effect of such parameters do not take into account the whole shape of the law. Specific kernels should be designed in order  to map distributions into a reproducing kernel Hilbert space  in which the whole arsenal of kernel methods can be extended to probability measures. This issue has recently been considered in \cite{2016arXiv160509522M} or \cite{KolouriZouRohde}. 

In the past few years, transport based distances such as the Monge-Kantorovich or Wasserstein distance have become a growing way to assess similarity between probability measures  and are used for numerous applications in learning and forecast problems. Since such distances are defined as a cost to transport one distribution to the other one, they appear to be a very relevant way to measure similarities between probability measures. Details on Wasserstein distances and their links with optimal transport problems can be found in ~\cite{villani2009optimal}.  Applications in statistics are developed in \cite{MR1625620}, \cite{MR3338645,Gouic2016}  while kernels have been developed in  \cite{KolouriZouRohde} or \cite{peyre2016gromov}. \vskip .1in

In this paper, we construct covariance functions in order to obtain Gaussian processes indexed by  probability measures. We provide a class of covariances which are functions of the Monge-Kantorovich distance, corresponding to stationary Gaussian processes. We also give covariances corresponding to the fractional Brownian processes indexed by probability distributions, which have stationary increments with respect to the Monge-Kantorovich distance. Furthermore we show original nondegeneracy results for these kernels.
Then, in this framework, we focus on the selection of a stationary covariance kernel in a parametric model through maximum likelihood. We prove the consistency and asymptotic normality of the covariance parameter estimators. We then consider the Kriging of such Gaussian processes. We prove the asymptotic accuracy of the Kriging prediction under the estimated covariance parameters. In simulations, we show the strong benefit of the studied kernels, compared to more standard kernels operating on finite dimensional projections of the distributions. In addition, we show in the simulations that the Gaussian process model suggested in this article is significantly more accurate that the kernel smoothing based predictor of \cite{poczos13distribution}.
Our results consolidate the idea that the Monge-Kantorovich distance is an efficient tool to assess variability between distributions, leading to sharp predictions of the outcome of a Gaussian process with distribution-type inputs.  \vskip .1in

The paper falls into the following parts. In Section \ref{section:generalities} we recall generalities on the Wasserstein space, covariance kernels and stationarity of Gaussian processes. Section \ref{section:kernels} is devoted to the construction and analysis of an appropriate kernel for probability measures on $\mathbb{R}$. Asymptotic results on the estimation of the covariance function and properties of the prediction of the associated Gaussian process are presented in Section \ref{section:modeling}. Section \ref{section:simulation} is devoted to numerical applications while the proofs are postponed to the appendix.

\section{An applicative case from nuclear safety} \label{section:IRSN}

The research that conducted to this article have been partially funded by CEA, and is motivated by a nuclear safety application, which we detail here.

A standard problem for used fissile storage process is the axial
burn up analysis of fuel pins \cite{radulescu2009sensitivity}.
In this case study, fuel pins may be seen as one-dimensional curves
$X:[0,1] \to \mathbb{R}^+$ \cite{cacciapouti2000axial}. These curves correspond to the axial
irradiation profiles for fuel in transportation or storage packages
which define the neutronic reactivity of the systems. From a curve $X$,
corresponding to a given irradiation profile, it is then possible to
compute the resulting neutron multiplication factor $k_{eff}(X)$ by
numerical simulation \cite{bowman2003scale}.
It can be insightful, for profiles with a given total irradiation
$\int_{0}^1 X(t)dt$, to study the impact of the shape of the irradiation
curve $X$ on the multiplication factor $k_{eff}(X)$. This type of study
can be addressed by considering $k_{eff}$ as a realization of a Gaussian
process indexed by one-dimensional distributions. \vskip .1in
\section{\textbf{Generalities}} \label{section:generalities}
In this section we recall some basic definitions and properties of the Wasserstein spaces and of covariance kernels.\vskip .1in

\paragraph{The Monge-Kantorovich distance} \label{par:Wass}
Let us consider  the set $\W$ of probability measures on $\mathbb{R}$ with a finite moment of order two.
For two $\mu,\nu$ in $\mathcal{W}_2\left( \mathbb{R} \right)$ , we denote by $\Pi(\mu, \nu)$ the set of all
probability measures $\pi$ over the product set $\mathbb{R} \times\mathbb{R} $
with first (resp. second) marginal $\mu$ (resp. $\nu$).

The transportation cost with quadratic cost function, or quadratic transportation cost,
between these two measures
$\mu$ and $\nu$   is defined as

\begin{equation} \label{eq:quadratic_cost}
	\mathcal{T}_2(\mu, \nu) = \inf_{ \pi \in \Pi(\mu, \nu)} \int \left| x - y\right| ^2 d \pi(x,y).
\end{equation}
This transportation cost allows to endow the set $\mathcal{W}_2\left(\mathbb{R}\right)$
with a metric by defining  the quadratic  Monge-Kantorovich, or quadratic Wasserstein distance between $\mu$ and $\nu$ as
\begin{equation} \label{eq:wasserstein_distance}
 W_2(\mu, \nu)  = \mathcal{T}_2(\mu, \nu) ^{1/2}.
\end{equation}
A probability measure $\pi$ in $\Pi(\mu,\nu)$ realizing the infimum in \eqref{eq:quadratic_cost} is called an optimal coupling.
This vocabulary transfers to a random vector $(X_1,X_2)$ with distribution $\pi$.
We will call $\W$ endowed with the distance $W_2$ the Wasserstein space.

We will consider on several occasions the collection of random variables $(F^{-1}_\mu(U))_{\mu \in \W}$, where $F^{-1}_\mu$ defined as 
$$  F^{-1}_\mu(t) = \inf \{u, F_\mu(u) \geq t \}$$ denotes the quantile function of the distribution $\mu$, and $U$ is an uniform random variable on $[0,1]$. For every $\mu,\nu \in \W$, the random vector $((F^{-1}_\mu(U)),(F^{-1}_\nu(U)))$ is an optimal coupling (see \cite{villani2009optimal}). Notice that the random variable $F^{-1}_\mu(U)$ does not depend on $\nu$, so that $(F^{-1}_\mu(U))_{\mu \in \W}$ is an optimal coupling between every distribution of $\W$.

More details on Wasserstein distances and their links with optimal transport problems can be found in  
\cite{rachev} or \cite{villani2009optimal} for instance. \vskip .1in

\paragraph{Covariance kernels}
Let us recall that the law of a Gaussian random process $(X(x))_{x\in E}$ indexed by a set $E$ is entirely characterized by its mean and covariance functions
$$ M: x \mapsto \E (X(x))$$
and $$K: (x,y) \mapsto \Cov(X(x) X(y))$$
(see $e.g.$ \cite{lifshits2012lectures}).

A function $K$ is actually the covariance of a random process if and only if it is a \emph{positive definite kernel}, that is to say for every $x_1,\cdots,x_n \in E$ and $\lambda_1,\cdots,\lambda_n \in \mathbb{R}, $
\begin{equation} \label{eq:ineq_positive_definite} \sum_{i,j=1}^n \lambda_i \lambda_j K(x_i,x_j) \geq 0. \end{equation}

In this case we say that $K$ is a \emph{ covariance kernel}.

On the other hand, any function can be chosen as the  mean of a random process. Hence without loss of generality we focus on centered random processes in Section \ref{section:kernels}.

Positive definite kernels are closely related to negative definite kernels. A function $K:E \times E \rightarrow \mathbb{R}$ is said to be a \emph{negative definite kernel} if for every $x\in E$,
\begin{equation} K(x,x)=0\end{equation}
and for every $x_1,\cdots,x_n \in E$ and $c_1,\cdots,c_n \in \mathbb{R}$ such that $\sum_{i=1}^n c_i =0$,

\begin{equation} \label{eq:ineq_negative_definite} \sum_{i,j=1}^n c_i c_j K(x_i,x_j) \leq 0. \end{equation}

\begin{example} The variogram $(x,y)\mapsto \E(X(x)-X(y))^2$ of any random field $X$ is a negative definite kernel. \end{example}

If the inequality \eqref{eq:ineq_positive_definite} (resp. \eqref{eq:ineq_negative_definite}) is strict as soon as not every $\lambda_i$ (resp. $c_i$) is null and the $x_i$ are two by two distinct, a positive definite (resp. negative definite) kernel is said to be \emph{nondegenerate}. Nondegeneracy of a covariance kernel is equivalent to the fact that every covariance matrix built with $K$ is invertible. We will say that a Gaussian random process is nondegenerate  if its covariance function is a nondegenerate kernel. Nondegeneracy is is usually a desirable condition for Kriging, since the forecast is built using the inverse of the covariance matrix of the observations. In addition, Gaussian process models with degenerate kernels have structural restrictions that can prevent them for being flexible enough. We give the nondegeneracy of the fractional Brownian motion indexed by the Wasserstein space in Section \ref{section:kernels}.\vskip .1in

\paragraph{Stationarity} Stationarity is a property of random processes that is standard in the Kriging literature. Roughly speaking, a stationary random process behaves in the same way at every point of the index space. It is also an enjoyable property for technical reasons. In particular it is a key assumption for the proofs of the properties we give in Section \ref{section:modeling}.

We say that a random process $X$ indexed by a metric space $(E,d)$ is \emph{stationary} if it has constant mean and for every isometry $g$ of the metric space we have
\begin{equation} \label{eq:stationary}  \Cov(X(g(x)),X(g(y)))=\Cov(X(x),X(y)).\end{equation}

Let us notice in particular that if the covariance of a random process  is a function of the distance, equation \eqref{eq:stationary} is verified. This is the assumption we make in Section \ref{section:modeling}.

One can also find the assumption of stationarity for the increments of a random process. Many classical random processes have stationarity increments, such as the fractional Brownian motion. We prove the existence of fractional Brownian motion indexed by the Wasserstein space in Section \ref{section:kernels}.

We will say that $X$ has \emph{stationary increments} starting in $o \in E$ if $X$ is centred, $X(o)=0$ almost surely, and for every isometry $g$ we have
\begin{equation}\Cov\left(X(g(x))-X(g(o))\right)=\Cov\left(X(x)-X(o)\right).\end{equation}

Notice that the variance of a random process with stationary increments increases as the input gets far from the origin point~$o$.

Let us remark that the definitions we gave are usually called ``in the wide sense", in contrast with stationarity definitions ``in the strict sense", which asks for the law of the process (or its increments) to be invariant under the action of the isometries, and not only the first and second moments. Since we are only dealing with Gaussian processes those definitions coincide. \vskip .1in

\paragraph{Isometries of the Wasserstein space}

Since we are interested in processes indexed by the Wasserstein space with stationarity features, let us recall a few facts about isometries of the Wasserstein space $\W$, that is to say maps $i:\W \rightarrow \W$ that preserve the Wasserstein distance.

\emph{Trivial isometries} come from isometries of $\mathbb{R}$: to any isometry  $g:\mathbb{R}\rightarrow \mathbb{R}$, we can associate an isometry $g_{\#}: \W \rightarrow \W$ that maps any measure $\mu \in \W$ to the measure
$$g_{\#}(\mu): A \mapsto \mu(g^{-1}(A)).$$
Stationarity of a random process with regard to these trivial isometries is an interesting feature, since it means that the statistical properties of the outputs do not change when we apply an isometry to the real line.

However let us mention that not every isometry of the Wasserstein space is trivial. In particular, mapping every distribution to its symmetric regarding its expectancy defines an isometry of $\W$. We refer to \cite{kloeckner2010geometric} for a complete description of the isometries of the Wasserstein space.

\section{\textbf{Gaussian process models for distribution inputs}} \label{section:kernels}

In this section we give covariance kernels on the space of probability distributions on the real line. This allows for modeling and Gaussian process regression of datasets with distribution inputs.

We start in Section \ref{subsec:fBm_kernels}  by giving a generalization of the seminal fractional Brownian motion to distributions inputs endowed with the Wasserstein distance.

Then, in Section \ref{subsec:stationary_kernels} we give Gaussian processes that are stationary with respect to the Wasserstein distance on the inputs.

\subsection{Fractional Brownian motion with distribution inputs} \label{subsec:fBm_kernels}
We first consider the family of \emph{fractional Brownian kernels}
\begin{multline} \label{eq:fbm_kernels} K^{H,\mu_0}(\mu,\nu) \\ = \frac{1}{2}\left(W_2^{2H}(\mu_0,\mu)+W_2^{2H}(\mu_0,\nu)-W_2^{2H}(\mu,\nu) \right), \end{multline}
where $0<H\leq 1$ and $\mu_0 \in \W$ are fixed.

Note that these kernels are obtained by taking the covariances of the classical fractional Brownian motions and replacing the distance $|t-s|$ between two times $s,t\in \mathbb{R}$ by the Wasserstein distance $W_2(\mu,\nu)$ between two distribution inputs. The measure $\mu_0 \in \W$ plays the role of the origin $0\in \bbR$.

\begin{theorem} \label{thm:fractional_brownian_kernels} For every $0 \leq H \leq 1$ and a given $\mu_0 \in \W$ the function
	$K^{H,\mu_0}$ defined by \eqref{eq:fbm_kernels} is a  covariance function on $\W$. Furthermore $K^{H,\mu_0}$ is nondegenerate if and only if $0<H<1$.
\end{theorem}
The Gaussian process $(X(\mu))_{\mu \in \W}$ such that
\begin{equation} \label{eq:carac_fBm}
	\left\{
	\begin{aligned}
		\E X(\mu) &= 0,\\
		\Cov (X(\mu),X(\nu)) &=K^{H,\mu_0}(\mu,\nu)
	\end{aligned}
	\right.
\end{equation} is the \emph{$H$-fractional Brownian motion with index space $\W$ and origin in $\mu_0$.} It inherits properties from the classical fractional Brownian motion.

 It is easy to check that the output at the origin measure $\mu_0$ is zero, $X(\mu_0) = 0 \text{ almost surely.}$ Furthermore
\begin{equation} \label{eq:carac_fBm2}
		\E (X(\mu)-X(\nu))^2=W_2^{2H}(\mu,\nu),\\
\end{equation}
from which we deduce that $(X(\mu))_{\mu \in \W}$ has stationary increments, which means that the statistical properties of $X(\mu)-X(\nu)$ are the same as those of $X(g(\mu))-X(g(\nu))$ for every isometry $g$ of the Wasserstein space. 

The  fractional Brownion motion is well known for its parameter $H$ governing the regularity of the trajectories: small values of $H$ correspond to very irregular trajectories while greater values give steadier paths. Moreover for $H>1/2$ the process exhibits long-range dependence (see \cite{mandelbrot1968fractional}).

From the modelling point of view, it is interesting to consider the following process: consider $(X(\mu))_{\mu \in \W}$ the $H$-fractional Brownian motion with origin in $\delta_0$ the Dirac measure at $0$, $f$ a real-valued function and define
\begin{equation}Y(\mu):=X(\bar{\mu})+f(m(\mu)),\end{equation}

where $\bar{\mu}$ denotes the centred version of $\mu$.
We then have, using $X(\delta_0)=0$ almost surely and \eqref{eq:carac_fBm2}:
\begin{align*}\Var(Y(\mu))=\E(X(\bar{\mu}))^2  =\E(X(\bar{\mu})-X(\delta_0))^2
& = W_2^{2H}(\bar{\mu},\delta_0)\\
& = \left(\Var(\mu)\right)^H.
\end{align*}
Hence the mean of the output $Y(\mu)$ is a function of the mean of the input distribution $\mu$ and its dispersion is an increasing function of the dispersion of $\mu$. This is a valuable property when modeling a function $\mu \mapsto g(\mu)$ as a Gaussian Process realization $\mu \mapsto Y(\mu)$, when it is believed that the range of possible values for $g$ increases with the variance of the input $\mu$.
%

Let us further notice that for $f=id$ and $H=1$ we have $$\E(Y(\mu))=\E(F^{-1}_\mu(U))$$ and $$\Cov(Y(\mu), Y(\nu))=\Cov(F^{-1}_\mu(U),F^{-1}_{\mu}(U)),$$ where $F^{-1}_\mu$ denotes the quantile function of the distribution $\mu$, and $U$ is an uniform random variable on $[0,1]$. In some sense, $Y$ is in this case the Gaussian process that mimics the statistical properties of the optimal coupling $(F^{-1}_\mu(U))_{\mu \in \W}$ (see Section \ref{section:generalities} a ). 


%
%
%
%

From now on (with the exception of Section \ref{subsec:proofs_kernels} from the appendix where we prove Theorem \ref{thm:fractional_brownian_kernels}) we will focus on stationary processes, which are more adapted to learning tasks on distributions where there is no a priori reason to associate different dispersion properties to the outputs corresponding to different distribution inputs.

\subsection{Stationary processes} \label{subsec:stationary_kernels}
We now construct Gaussian processes which are stationary with respect to the Wasserstein distance.

\begin{theorem} \label{thm:stationary_valid_kernels} For every completely monotone function $F$ and $0 < H \leq 1$ the function \begin{equation} \label{eq:stationnary_covariance} (\mu,\nu)\mapsto F\left(W^{2H}_2(\mu,\nu)\right)\end{equation} is a covariance function on $\W$. Furthermore a Gaussian random process with constant mean and covariance \eqref{eq:stationnary_covariance} is stationary with respect to the Wasserstein distance.
\end{theorem}

We recall that a $\mathcal{C}^\infty$ function $F:\mathbb{R}^+ \rightarrow \mathbb{R}^+$ is said to be \emph{completely monotone} if for every $n \in \mathbb{N}$ and $x \in \mathbb{R}^+$,
$$(-1)^n F^{(n)}(x) \geq 0.$$ Here $F^{(n)}$ denotes the derivative of order $n$ of $F$.
The prototype of a completely monotone fuction is $x\mapsto e^{-\lambda x}$, for any positive $\lambda$. Furthermore $F$ is completely monotone if and only if it is the Laplace transform of a positive measure $\mu_F$ with finite mass on $\mathbb{R}^+$, that is to say
$$F(x)=\int_{\mathbb{R}^+} e^{-\lambda x} d\mu_F(\lambda). $$
Other examples of completely monotone functions include $x^{-\lambda}$ for positive values of $\lambda$ and $\log\left(1+\frac{1}{x}\right)$.

\begin{example} \label{example:stationary:kernels} Applying theorem \ref{thm:stationary_valid_kernels} with the completely monotone functions $e^{-\lambda x}$ we obtain the stationary covariance kernels 
	\begin{equation} \label{eq:exp_kernels} e^{-\lambda W_2^{2H}(\mu,\nu)},\end{equation} for every $\lambda>0$ and $0 < H \leq 1$.
	
	These kernels are generalizations to distribution inputs of the kernels of the form $e^{-\lambda \| x-y\|^{2H}}$ on $\mathbb{R}^d$, which are classical in spatial statistics and machine learning. In particular setting $H=1/2$ gives the family of  \emph{Laplace kernels}, and $H=1$ the family of \emph{Gaussian kernels}.
	
\end{example}

At this point we have obtained enough  covariance functions to consider parametric models that fit practical datasets. Section \ref{section:modeling} addresses the question of the selection of the best covariance kernel amongst a parametric family of stationary kernels, together with the prediction of the associated Gaussian process. In Section \ref{section:simulation} we carry out simulations with the following parametric model, which is directly derived from \eqref{eq:exp_kernels}:

\begin{equation}\label{eq:exponential_model}\left\{K_{\sigma^2,\ell,H}=\sigma^2 e^{-\frac{W_2^{2H}}{\ell}}, ~ (\sigma^2,\ell,H) \in C\times C' \times [0,1] \right\}, \end{equation}
where $C,C' \subset (0,\infty)$ are two compact sets.

\subsection{Ideas of proof} \label{subsec:ideas}

Theorems \ref{thm:fractional_brownian_kernels} and \ref{thm:stationary_valid_kernels} are direct corollaries of the following result:

\begin{theorem} \label{thm:negative_definite} The function $W_2^{2H}$ is a negative definite kernel if and only if $0\leq H\leq 1$. Furthermore, it is nondegenerate if and only if $0 <H <1$.
\end{theorem}

One can find in \cite{KolouriZouRohde} a proof of the negative definiteness of the kernel $W_2^{2H}$ restricted to absolutely continuous distributions in $\mathcal{W}_2(\mathbb{R})$.
The proof given here holds for any distribution of $\mathcal{W}_2(\mathbb{R})$, and we provide the nondegeneracy property of the kernel.

In short (see Appendix \ref{subsec:proofs_kernels} for a detailed proof), we consider $H=1$ and the optimal coupling (see Section \ref{section:generalities} a)
\begin{equation} (Z(\mu))_{\mu \in \W}:=(F^{-1}_{\mu}(U))_{\mu \in \W},\end{equation}
where $F^{-1}_\mu$ is the quantile function of the distribution $\mu$ and $U$ is an uniform random variable on $[0,1]$.
This coupling can be seen as a (non-Gaussian!) random field indexed by $\W$.
As such, its variogram 
\begin{equation} (\mu,\nu) \mapsto \E(Z(\mu)-Z(\nu))^2 \end{equation}
is a negative definite kernel. Furthermore it is equal to $W_2^2(\mu,\nu)$ since the coupling $(Z(\mu))$ is optimal (see \eqref{eq:quadratic_cost}). The proof ends with the use of the following classical lemma:

\begin{lem} \label{lem:subordination} If $K$ is a negative definite kernel then $K^{H}$ is a negative definite kernel for every $0 \leq H\leq1$.
\end{lem}

See  $e.g.$ \cite{berg_al} for a proof  Lemma \ref{lem:subordination}.

\begin{rmq} In \cite{istas2011manifold}, Istas defines the fractional index of a metric space $E$ endowed with a distance $d$ by \begin{equation} \beta_E:=\sup \left\{ \beta>0 ~ | ~ d^{\beta}~\text{is negative definite}\right\}.\end{equation}
	
	 One of the interpretation of the fractional index is that $\beta_E/2$ it is the maximal regularity for a fractional Brownian motion indexed by $(E,d)$: indeed the $H$-fractional Brownian motion indexed by a metric space exists if and only if $H\leq \beta_E /2$. For instance, the fractional exponent of the Euclidean spaces $\mathbb{R}^n$ is equal to $2$, while the fractional index of the spheres $\mathbb{S}^n$ is only $1$. Recall that an $H$-fractional Brownian motion has more regular paths and exhibits long-distance correlation for large values of $H$. In a non-rigorous way, the fractional index can be seen as some measure of the difficulty to construct long-distance correlated random fields indexed by the space $(E,d)$.
	 
	 It is in general a difficult problem to find the fractional index of a given space. Theorem \ref{thm:negative_definite} states that the fractional exponent $\beta_{\W}$ of the Wasserstein space is equal to $2$.
\end{rmq}

\section{\textbf{Model selection and Gaussian process regression}} \label{section:modeling}

\subsection{Maximum Likelihood and prediction}

Let us consider a Gaussian process $Y$ indexed by $\mathcal{W}_2(\mathbb{R})$, with zero mean function and unknown covariance function $K_0$. Most classically, it is assumed that the covariance function $K_0$ belongs to a parametric set of the form
\begin{equation} \label{eq:parametric_model_modeling}
\{ K_{\theta} ; \theta \in \Theta \},
\end{equation}
with $\Theta \subset \mathbb{R}^p$ and where $K_{\theta}$ is a  covariance function and $\theta$ is called the covariance parameter.
Hence we have $K_0=K_{\theta_0}$ for some true parameter $\theta_0 \in \Theta$.

For instance, considering the fractional Brownian motion kernel given in \eqref{eq:fbm_kernels}, we can have $\theta = (\sigma^2,H)$, $\Theta = (0,\infty) \times (0,1]$ and $K_{\theta} = \sigma^2 K^{H,\eta}$, where $\eta$ is fixed in $\W$. In this case, the covariance parameters are the order of magnitude parameter $\sigma^2$ and the regularity parameter $H$.

Typically, the covariance parameter $\theta$ is selected from a data set of the form $(\mu_i,y_i)_{i=1,...,n}$, with $y_i = Y(\mu_i)$. Several techniques have been proposed for constructing an estimator $\hat{\theta} = \hat{\theta}(\mu_1,y_1,...,\mu_n,y_n)$, in particular maximum likelihood (see e.g. \cite{stein99interpolation}) and cross validation \cite{bachoc13cross,bachoc16asymptotic,zhang10kriging}.
In this paper, we shall focus on maximum likelihood, which is widely used in practice and has received a lot of theoretical attention.

Maximum Likelihood is based on maximizing the Gaussian likelihood of the vector of observations $(y_1,...,y_n)$. The estimator is $\ML \in \argmin L_\theta$ with
\begin{equation} \label{eq:ML}
L_\theta= \frac{1}{n} \ln (\det R_\theta) +\frac{1}{n}y^t R_\theta^{-1} y,
\end{equation}
where $R_\theta = [ K_{\theta}(\mu_i , \mu_j) ]_{1 \leq i,j \leq n}$

Given the maximum likelihood estimator $\ML$, the value $Y(\mu)$, for any input $\mu \in \mathcal{W}_2(\mathbb{R})$, can be predicted by plugging (see for instance in~\cite{stein99interpolation}) $\ML$ in the conditional expectation (or posterior mean) expression for Gaussian processes. More precisely, $Y(\mu)$ is predicted by $\yc (\mu)$
with 
\begin{equation} \label{eq:pred}
\hat{Y}_\theta(\mu)=r_\theta^t(\mu)R_\theta^{-1} y
\end{equation}
and 
\[
r_\theta(\mu)=\left[\begin{array}{c} K_\theta(\mu,\mu_1)\\ \vdots\\ K_\theta(\mu,\mu_n) \end{array}\right].
\]
Note that $\hat{Y}_\theta(\mu)$ is the conditional expectation of $Y(\mu)$ given $y_1,...,y_n$, when assuming that $Y$ is a centered Gaussian process with covariance function $K_{\theta}$.

\subsection{Asymptotic properties} \label{subsection:asymptotic:properties}

In this section, we aim at showing that some of the asymptotic results of the Gaussian process literature, which hold for Gaussian processes indexed by $\mathbb{R}^d$, can be extended to Gaussian processes indexed by $\W$. To our knowledge, this extension has not been considered before.

For a Gaussian process indexed by $\mathbb{R}^d$, two main asymptotic frameworks are under consideration: fixed-domain and increasing-domain asymptotics \cite{stein99interpolation}. Under increasing-domain asymptotics, as $n \to \infty$, the observation points $x_1,...,x_n \in \mathbb{R}^d$ are so that $\min_{i \neq j} || x_i - x_j ||$ is lower bounded. Under fixed-domain asymptotics, the sequence (or triangular array) of observation points $(x_1,...,x_n)$ becomes dense in a fixed bounded subset of $\mathbb{R}^d$. To be specific, for a Gaussian process indexed by $\mathbb{R}$, a standard increasing-domain framework would be given by $x_i = i$ for $i \in \mathbb{N}$, while a standard fixed-domain framework would be given by, for $n \in \mathbb{N}$, $x_i = i/n$ for $i=1,...,n$. 

Let us now briefly review the existing results for Gaussian processes indexed by $\mathbb{R}^d$.
Typically, under increasing-domain asymptotics, the true covariance parameter $\theta_0$ is estimated consistently by maximum likelihood, with asymptotic normality \cite{MarMar1984,cressie93asymptotic,cressie96asymptotics,
shaby12tapered,bachoc14asymptotic,furrer16asymptotic}. Also, predicting with the estimated covariance parameter $\hat{\theta}$ is asymptotically as good as predicting with $\theta_0$ \cite{bachoc14asymptotic}.

Under fixed-domain asymptotics, there are cases where some components of the true covariance parameter $\theta_0$ can not be consistently estimated  \cite{stein99interpolation,zhang04inconsistent}.
Nevertheless, these components which can not be estimated consistently do not have an asymptotic impact on prediction \cite{AEPRFMCF,BELPUICF,UAOLPRFUISOS}. Some results on prediction with estimated covariance parameters are available in \cite{putter01effect}. Also, asymptotic properties of maximum likelihood estimators are obtained in \cite{Yin1991,Yin1993,CheSimYin2000,ESCMSGRFM,Loh2005}. 

We remark, finally,
that the above increasing-domain asymptotic results hold for fairly general classes of covariance functions, while fixed-domain asymptotic results currently have to be derived for specific covariance functions and on a case-by-case basis. 

For this reason, in this paper, we focus on extending some of the above increasing-domain asymptotic results to Gaussian processes indexed by $\W$. Indeed, this will enable us to obtain a fair amount of generality with respect to the type of covariance functions considered.

We thus extend the contributions of \cite{bachoc14asymptotic} in the case of Gaussian processes with probability distribution inputs. In the rest of the section, we first list and discuss technical conditions for the asymptotic results. Then, we show the consistency and asymptotic normality of maximum likelihood and show that predictions from the maximum likelihood estimator are asymptotically as good as those obtained from the true covariance parameter. In Section \ref{subsection:asymptotics:example}, we study an explicit example, for which all the technical conditions can be satisfied. All the proofs are postponed to the appendix. At the end of Section \ref{subsection:asymptotics:example}, we discuss the novelty of these proofs, compared to those of the literature, and especially those in \cite{bachoc14asymptotic}.

The technical conditions for this section are listed below.

\begin{condition} \label{cond:asymptotics:un}
We consider a triangular array of observation points $\{\mu_1,...,\mu_n\} = \{\mu_1^{(n)},...,\mu_n^{(n)}\}$
so that for all $n \in \mathbb{N}$ and $1 \leq i \leq n$, $\mu_i$ has support in $[i,i+L]$ with a fixed $L <\infty$. 
\end{condition}
\begin{condition} \label{cond:asymptotics:deux}
The model of covariance functions $\{K_\theta, \theta \in \Theta \}$ satisfies
			$$\forall \theta \in \Theta, ~ K_\theta(\mu,\nu) = F_\theta \left(W_2(\mu, \nu) \right), $$
with $F_{\theta}: \mathbb{R}^+ \to \mathbb{R}$ 
and
$$\sup_{\theta \in \Theta} \left|F_\theta(t) \right| \leq \frac{A}{1+|t|^{1+\tau}} $$
with a fixed $A< \infty$, $\tau >1$.
\end{condition}
\begin{condition} \label{cond:asymptotics:trois}
We have observations $y_i=Y(\mu_i)$, $i=1,\cdots,n$ of the centered Gaussian process $Y$ with covariance function $K_{\theta_0}$ for some $\theta_0\in \Theta$.
\end{condition}
\begin{condition} \label{cond:asymptotics:quatre}
The sequence of matrices $R_\theta=\left(K_\theta (\mu_i, \mu_j) \right)_{1\leq i,j \leq n}$ satisfies
		$$\lambda_{\inf}(R_\theta) \geq c $$ for a fixed $c>0$, where $\lambda_{\inf}(R_\theta)$ denotes the smallest eigenvalue of $R_\theta$.
\end{condition}
\begin{condition} \label{cond:asymptotics:cinq}
$\forall \alpha >0$, $$\liminf \limits_{n\rightarrow \infty} \inf\limits_{\|\theta-\theta_0 \| \geq \alpha} \frac{1}{n} \sum_{i,j=1}^n \left[K_\theta(\mu_i,\mu_j)-K_{\theta_0}(\mu_i,\mu_j) \right]^2 >0.$$
\end{condition}
\begin{condition} \label{cond:asymptotics:six}
$\forall t \geq 0$, $F_\theta(t)$ is continuously differentiable with respect to $\theta$ and we have
		$$\sup_{\theta\in\Theta} \max_{i=1,\cdots,p} \left| \frac{\partial}{\partial \theta_i}F_\theta (t) \right| \leq \frac{A}{1+t^{1+\tau}}, $$ with $A,\tau$ as in Condition \ref{cond:asymptotics:deux}.
\end{condition}
\begin{condition} \label{cond:asymptotics:sept}
$\forall t \geq 0$, $F_\theta(t)$ is three times continuously differentiable with respect to $\theta$ and we have,
for $q\in \{2,3\}$, $i_1\cdots i_q \in \{1,\cdots p\}$,
			$$\sup_{\theta\in\Theta}   \left| \frac{\partial}{\partial \theta_{i_1}} \cdots \frac{\partial}{\partial \theta_{i_q}} F_\theta (t) \right| \leq \frac{A}{1+t^{1+\tau}}. $$
\end{condition}
\begin{condition} \label{cond:asymptotics:huit}
$\forall (\lambda_1\cdots, \lambda_p) \neq (0,\cdots,0)$,
		$$\liminf\limits_{n\rightarrow \infty} \frac{1}{n} \sum_{i,j=1}^n
		 \left( \sum_{k=1}^p \lambda_k \frac{\partial}{\partial_{\theta_k}} K_{\theta_0} \left(\mu_i,\mu_j \right) \right)^2>0.$$
\end{condition}

Condition \ref{cond:asymptotics:un} mimics the increasing-domain asymptotic framework discussed above for vectorial inputs. In particular, the observation measures $\mu_i$ and $\mu_j$ yield a large Wasserstein distance when $|i - j|$ is large.

Condition \ref{cond:asymptotics:deux} entails that all the covariance functions under consideration are stationary in the sense that the covariance between $\mu$ and $\nu$ depends only on the distance $W_2(\mu,\nu)$. Stationarity is also assumed when considering increasing-domain asymptotics for Gaussian processes indexed by $\mathbb{R}^d$ \cite{MarMar1984,cressie93asymptotic,cressie96asymptotics,
shaby12tapered,bachoc14asymptotic,furrer16asymptotic}.
Hence, we remark that the asymptotic results of the present section do not apply to the covariance functions of fractional Brownian motion in \eqref{eq:fbm_kernels}. On the other hand, these results apply to the power exponential covariance functions in \eqref{eq:exponential_model}.

Condition \ref{cond:asymptotics:deux} also imposes that the covariance functions in the parametric model decrease fast enough with the Wasserstein distance. This condition is standard in the case of vector inputs, and holds for instance for the covariance functions in \eqref{eq:exponential_model}.

Condition \ref{cond:asymptotics:trois} means that we address the well-specified case \cite{bachoc13cross,bachoc16asymptotic}, where there is a true covariance parameter $\theta_0$ to estimate.

Condition \ref{cond:asymptotics:quatre} is technically necessary for the proof techniques of this paper. This condition holds whenever the covariance model satisfies, for all $\theta \in \Theta, w \geq 0$, $F_{\theta}(w) = \bar{F}_{\theta}(w) + \delta_{\theta} \mathbf{1}_{ \{ w=0 \} }$, where $\bar{F}_{\theta}$ is a continuous covariance function and where $\inf_{\theta \in \Theta}\delta_{\theta} > 0$. This situation corresponds to Gaussian processes observed with Gaussian measure errors, or to Gaussian processes with very small scale irregularities, and is thus representative of a significant range of practical applications. 

In the case where $F_{\theta}$ is continuous (which usually means that we have exact observations of a Gaussian process with continuous realizations), then Condition \ref{cond:asymptotics:quatre} implies that
\begin{equation} \label{eq:bounded:away:zero:distances}
\inf_{n \in \mathbb{N},i \neq j =1,...,n} W_2( \mu_i , \mu_j ) >0.
\end{equation}
For a large class of Gaussian processes indexed by $\mathbb{R}^d$, it has been shown that the condition in \eqref{eq:bounded:away:zero:distances} (with $W_2$ replaced by the Euclidean distance) is also sufficient for Condition \ref{cond:asymptotics:quatre} \cite{bachoc14asymptotic,bachoc2016smallest}. The proof relies on the Fourier transform on $\mathbb{R}^d$. For Gaussian processes indexed by $\W$, one could expect the condition in \eqref{eq:bounded:away:zero:distances} to be sufficient to guarantee \ref{cond:asymptotics:quatre} in many cases, although, to our knowledge, obtaining rigorous proofs in this direction is an open problem.

Condition \ref{cond:asymptotics:cinq} means that there is enough information in the triangular array $\{\mu_1,...,\mu_n\}$ to differentiate between the covariance functions $K_{\theta_0}$ and $K_{\theta}$, when $\theta$ is bounded away from $\theta_0$. We believe that Condition \ref{cond:asymptotics:cinq} can be checked for specific explicit instances of the triangular array $\{\mu_1,...,\mu_n\}$, as it involves an explicit sum of covariance values. 

Conditions \ref{cond:asymptotics:six} and \ref{cond:asymptotics:sept} are standard regularity and asymptotic decorrelation conditions for the covariance model. They hold, in particular, for the power exponential covariance model of \eqref{eq:exponential_model}.

Finally, Condition \ref{cond:asymptotics:huit} is interpreted as an asymptotic local linear independence of the $p$ derivatives of the covariance function, around $\theta_0$. Since this condition involves an explicit sum of covariance function derivatives, we believe that it can be checked for specific instances of the triangular array $\{\mu_1,...,\mu_n\}$.

We now provide the first result of this section, showing that the maximum likelihood estimator is asymptotically consistent.

	\begin{theorem} \label{theorem:consistency} 
	Let $\ML$ be as in \eqref{eq:ML}. Under Conditions \ref{cond:asymptotics:un} to \ref{cond:asymptotics:cinq}, we have as $n\rightarrow \infty$		
		$$\ML \overset{\bbP}{\longrightarrow} \theta_0. $$
	\end{theorem}
	
	In the next theorem, we show that the maximum likelihood estimator is asymptotically Gaussian. In addition, the rate of convergence is $\sqrt{n}$, and the asymptotic covariance matrix $M_{ML}^{-1}$ of $\sqrt{n}( \ML - \theta_0 )$ (that may depend on $n$) is asymptotically bounded and invertible, see \eqref{eq:asymptotic:cov:mat:bounded}.

	\begin{theorem} \label{theorem:TCL} Let $M_{ML}$ be the $p\times p$ matrix defined by
		$$(M_{ML})_{i,j}= \frac{1}{2n} \Tr \left( R_{\theta_0}^{-1} \frac{\partial R_{\theta_0}}{\partial \theta_i} R_{\theta_0}^{-1} \frac{\partial R_{\theta_0}}{\partial \theta_j}\right),$$
		with $R_{\theta}$ as in \eqref{eq:ML}.
		Under Conditions \ref{cond:asymptotics:un} to
\ref{cond:asymptotics:huit}	we have
		$$\sqrt{n} M_{ML}^{1/2} \left(\ML - \theta_0\right) {\overset{\mathcal{L}}{\underset{n \rightarrow \infty}{\longrightarrow}}} \mathcal{N}(0,I_n). $$
		Furthermore,
		\begin{equation} \label{eq:asymptotic:cov:mat:bounded}
		0 < \liminf_{n \to \infty} \lambda_{min} ( M_{ML} )
		\leq \limsup_{n \to \infty} \lambda_{max} ( M_{ML} ) < + \infty.
	\end{equation}
	\end{theorem}

In the next theorem, we show that, when using the maximum likelihood estimator, the corresponding predictions of the values of $Y$ are asymptotically equal to the predictions using the true covariance parameter $\theta_0$. Note that, in the increasing-domain framework considered here, the mean square prediction error is typically lower-bounded, even when using the true covariance parameter. Indeed, this occurs in the case of Gaussian processes with vector inputs, see Proposition 5.2 in \cite{bachoc14asymptotic}.

\begin{theorem} \label{theorem:input_prediction} Under 
Conditions \ref{cond:asymptotics:un} to \ref{cond:asymptotics:huit} we have

$$\forall \mu \in \W, \ \left| \yc (\mu)-\ycz (\mu) \right|=o_{\bbP}(1),$$
with $\hat{Y}_\theta(\mu)$ as in \eqref{eq:pred}.
\end{theorem}

\subsection{An example} \label{subsection:asymptotics:example}

In this section, we provide an explicit example of triangular array of probability measures for which Conditions \ref{cond:asymptotics:cinq} and \ref{cond:asymptotics:huit} are satisfied. We consider random probability measures $(\mu_i)_{i \in \mathbb{N}}$ which are independent and identically distributed (up to support shifts to satisfy condition \ref{cond:asymptotics:un}). We then show that Conditions \ref{cond:asymptotics:cinq} and \ref{cond:asymptotics:huit} are satisfied almost surely. The motivation for studying shifted independent and identically distributed random probability measures is that this this model is simple to describe and can generate a large range of sequences $\{ \mu_1,...,\mu_n \}$.

\begin{prop}  \label{prop:asymptotics:example}
Assume that Conditions \ref{cond:asymptotics:deux}, \ref{cond:asymptotics:six} and \ref{cond:asymptotics:sept} hold.

Assume that for $\theta \neq \theta_0$, $F_{\theta}$ and $F_{\theta_0}$ are not equal everywhere on $\mathbb{R}^+$. Assume that there does not exist $(\lambda_1,...,\lambda_p) \neq (0,...,0)$ so that $\sum_{i=1}^p (\partial / \partial \theta_i) F_{\theta_0}$ is the zero function on $\mathbb{R}^+.$

Let $(Z_i)_{i \in \mathbb{Z}}$ be independent and identically distributed Gaussian processes on $\mathbb{R}$ with continuous trajectories. Assume that $Z_0$ has mean function $0$ and covariance function $C_0$. Assume that $C_0(u,v) = C_0(u',v')$ whenever $v-u = v'-u'$ and let $C_0(u,v) = C_0(u-v)$ for ease of notation. Let $\hat{C}_0(w) = \int_{\mathbb{R}} C_0(t) e^{-\mathrm{i} w t} dt$ with $\mathrm{i}^2 = -1$. Assume that $\hat{C}_0(w) |w|^{2q}$ is bounded away from $0$ and $\infty$ as $|w| \to \infty$, for some fixed $q \in (0,\infty)$.

Let $L >1$ be fixed.
For $i \in \mathbb{Z}$, let $f_i: \mathbb{R}\to \mathbb{R}^+$ be defined by $f_i(t) = \exp(Z_i(t-i))/M_i$ if $t \in [i,i+L]$ and $f_i(t) = 0 $ else, where $M_i = \int_{i}^{i+L}  \exp(Z_i(t-i)) dt$. Let $\mu_i$ be the measure with probability density function $f_i$. Then, almost surely, with the sequence of random probability measures $\{\mu_1,...,\mu_n\}$, Conditions \ref{cond:asymptotics:cinq} and \ref{cond:asymptotics:huit} hold. 
\end{prop}

In Proposition \ref{prop:asymptotics:example}, the identifiability assumptions on $\{ F_{\theta} \}$ are very mild, and hold for instance for the power exponential model in \eqref{eq:exponential_model}. 

In Proposition \ref{prop:asymptotics:example}, the random probability measures have probability density functions obtained from exponentials of realizations of Gaussian processes. Hence, these measures have a non-parametric source of randomness, and can take flexible forms. Several standard covariance functions on $\mathbb{R}$ satisfy the conditions in Proposition \ref{prop:asymptotics:example}, in particular the Mat\'ern covariance functions (see e.g. \cite{stein99interpolation}).

We remark that, in the context of Proposition \ref{prop:asymptotics:example}, when $F_{\theta}(w) = \bar{F}_{\theta}(w) + \delta_{\theta} \mathbf{1}_{ \{ w=0 \} }$, with $\bar{F}_{\theta}$ a continuous covariance function and $\inf_{\theta \in \Theta}\delta_{\theta} > 0$, as described when discussing Condition \ref{cond:asymptotics:quatre}, then Conditions \ref{cond:asymptotics:un} to \ref{cond:asymptotics:huit} hold so that Theorems \ref{theorem:consistency}, \ref{theorem:TCL} and \ref{theorem:input_prediction} hold. If however $F_{\theta}$ is continuous, then Condition \ref{cond:asymptotics:quatre} almost surely does not hold since $L>1$ (as there will almost surely be pairs of distributions $\mu_i,\mu_i$, $i \neq j$, with arbitrarily small $W_2(\mu_i,\mu_j)$). Nevertheless, when $L<1$, it can be shown that Proposition \ref{prop:asymptotics:example} still holds when, in the conditions of this proposition on $\{ F_{\theta} \}$, $\mathbb{R}^+$ is replaced by $\cup_{i = 1 }^{\infty} [i-L,i+L]$. Also, as discussed above, when $L<1$, the condition in \eqref{eq:bounded:away:zero:distances} is satisfied and one could expect Condition \ref{cond:asymptotics:quatre} to hold.

We conclude this section by discussing the corresponding proofs (in the appendix). These proofs can be divided into two groups. In the first group (proofs of Theorems \ref{theorem:consistency}, \ref{theorem:TCL} and \ref{theorem:input_prediction} and of Proposition \ref{prop:for:proof:example}) we show that the arguments in \cite{bachoc14asymptotic} can be adapted and extended to the setting of the present article. The main innovations in this first group compared to \cite{bachoc14asymptotic} are that we allow for triangular arrays of observation points, and are not restricted to the specific structure of observation points of \cite{bachoc14asymptotic}. 

The proofs of the second group (proofs of Lemma \ref{lem:row_sum} and Proposition \ref{prop:asymptotics:example}) are specific to Gaussian processes with distribution inputs and are thus original for the most part. In particular, in the proof of Proposition \ref{prop:asymptotics:example}, we show that, for two measures obtained by taking exponentials of Gaussian processes, the corresponding random Wasserstein distance
has maximal distribution support.
In this aim, we use equivalence of Gaussian measure tools and specific technical manipulations of the Wasserstein distance.

\section{\textbf{Simulation study}} \label{section:simulation}

We now compare the Gaussian process model suggested in the present paper, with various models for predicting scalar outputs corresponding to distributional inputs. Among the covariance functions introduced in this paper, we shall focus on the power-exponential model \eqref{eq:exponential_model}, since its covariance functions are stationary with respect to the Wasserstein distance. We will not consider the fractional Brownian motion model \eqref{eq:fbm_kernels}, since it imposes to choose a ``zero distribution'', from which the variance increases with the distance. While this feature is relevant in some applications (for instance in finance), it is not natural in the simulation examples adressed here.

\subsection{Comparison with projection-based covariance functions} \label{subsection:comparison:other:cov:functions}

In this section, we focus on Gaussian process models for prediction. We compare the covariance functions \eqref{eq:exponential_model} of this paper, operating directly on the input probability distributions, to more classical covariance functions operating on projections of these probability measures on finite dimensional spaces.

\subsubsection{Overview of the simulation procedure}

We address the input-output map given by, for a distribution $\nu$ on $\mathbb{R}$, 
\[
F(\nu) = \frac{m_1(\nu)}{0.05 + \sqrt{ m_2(\nu) - m_1(\nu)^2 }},
\]
where $m_k(\nu) = \int_{\mathbb{R}} x^k d \nu(x)$.

We first simulate independently $n=100$ learning distributions $\nu_1,...,\nu_{100}$ as follows. First, we sample uniformly $\mu_i \in [0.3,0.7]$ and $\sigma_i \in [0.001,0.2]$, and compute $f_i$, the density of the Gaussian distribution with mean $\mu_i$ and variance $\sigma_i^2$.
Then, we generate the function $g_i$ with value $f_i(x) \exp( Z_i(x))$, $x\in [0,1]$, where $Z_i$ is a realization of a Gaussian process on $[0,1]$ with mean function $0$ and Mat\'ern $5/2$ covariance function with parameters $\sigma=1$ and $\ell = 0.2$ (see e.g. \cite{roustant12dicekriging} for the expression of this covariance function). Finally, $\nu_i$ is the distribution on $[0,1]$ having density $g_i / (\int_{0}^1 g_i)$. In Figure \ref{fig:sampled:dist}, we show the density functions of $10$ of these $n$ sampled distributions. From the figure, we see that the learning distributions keep a relatively strong underlying two dimensional structure, driven by the randomly generated means and standard deviations. At the same time, because of the random perturbations generated with the Gaussian processes $Z_i$, these distributions are not restricted in a finite-dimensional space, and can exhibit various degrees of asymmetries.

\begin{figure}
\centering
\includegraphics[height=7cm,width=7cm]{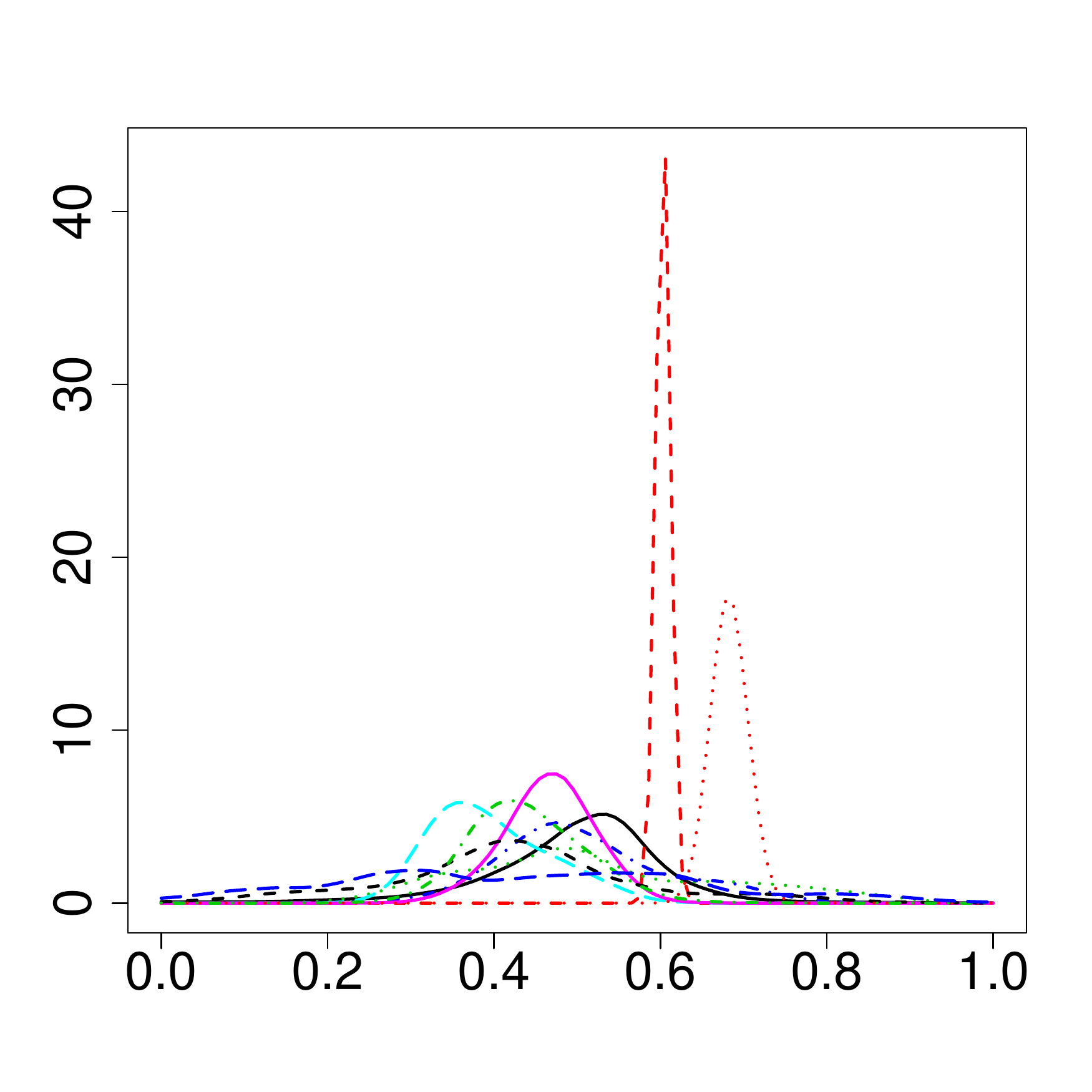}
\caption{Probability density functions of $10$ of the randomly generated learning distributions for the simulation study.}
\label{fig:sampled:dist}
\end{figure}

From the learning set $(\nu_i,F(\nu_i))_{i=1,...,n}$, we fit three Gaussian process models, which we call ``distribution'', ``Legendre'' and ``PCA'', and for which we provide more details below. Each of these three Gaussian process models provide a conditional expectation function $$\nu \to \hat{F}(\nu)  = \mathbb{E}( F( \nu )  | F(\nu_1),...,F(\nu_n) )$$ and a conditional variance function $$\nu \to \hat{\sigma}^2(\nu)  = \mathrm{var}( F( \nu )  | F(\nu_1),...,F(\nu_n) ).$$ We then evaluate the quality of the three Gaussian process models on a test set of size $n_t=500$ of the form $(\nu_{t,i},F(\nu_{t,i}))_{i=1,...,n_t}$, where the $\nu_{t,i}$ are generated in the same way as the $\nu_i$ above. We consider the two following quality criteria. The first one is the root mean square error (RMSE),
\[
RMSE^2 = \frac{1}{n_t} \sum_{i=1}^{n_t} \left( F(\nu_{t,i}) - \hat{F}(\nu_{t,i} ) \right)^2,
\]
which should be minimal. The second one is the confidence interval ratio (CIR) at level $\alpha \in (0,1)$,
\[
CIR_{\alpha} = \frac{1}{n_t} \sum_{i=1}^{n_t} \mathbf{1} \left\{ \left| F(\nu_{t,i}) - \hat{F}(\nu_{t,i}) \right| \leq q_{\alpha} \hat{\sigma}(\nu_{t,i}) \right\},
\]
with $q_{\alpha}$ the $\left(\frac{1}{2}+\frac{\alpha}{2}\right)$ quantile of the standard normal distribution. The $CIR_{\alpha}$ criterion should be close to $\alpha$.

\subsubsection{Details on the Gaussian process models}

The ``distribution'' Gaussian process model is based on the covariance functions discussed before, operating directly on probability distributions. In this model, the Gaussian process has mean function zero and a covariance function of the form
\[
K_{\sigma^2,\ell,H}( \nu_1 , \nu_2 ) = 
\sigma^2 \exp \left( 
- \frac{W_2( \nu_1 , \nu_2 )^{2H}}{\ell}
 \right).
\]
We call the covariance parameters $\sigma^2 > 0$, $\ell >0 $ and $H \in [0,1]$ the variance, correlation length and exponent. These parameters are estimated by maximum likelihood from the training set $(\nu_i,F(\nu_i))_{i=1,...,n}$, which yields the estimates $\hat{\sigma}^2,\hat{\ell},\hat{H}$. Finally, the Gaussian process model for which the conditional moments $\hat{F}(\nu)$ and $\hat{\sigma}^2(\nu)$ are computed is a Gaussian process with mean function zero and covariance function $K_{\hat{\sigma}^2,\hat{\ell},\hat{H}}$.

The ``Legendre'' and ``PCA'' Gaussian process models are based on covariance functions operating on finite-dimensional linear projections of the distributions. These projection-based covariance functions are used in the literature, in the general framework of stochastic processes with functional inputs, see e.g. \cite{Muehlenstaedt2016,nanty2016sampling}. For the ``Legendre'' covariance function, for a distribution $\nu$ with density $f_{\nu}$ and support $[0,1]$, we compute, for $i=0,...,o-1$
\[
a_i(\nu) = \int_{0}^1 f_{\nu}(t) p_i(t) dt,
\]
where $p_i$ is the $i-th$ normalized Legendre polynomial, with $\int_{0}^1 p_i^2(t) dt = 1$. The integer $o$ is called the order of the decomposition. Then, the covariance function operates on the input vector $(a_0(\nu),...,a_{o-1}(\nu))$ and is of the form
\begin{multline*}
 K_{\sigma^2,\ell_0,...,\ell_{o-1},H}( \nu_1 , \nu_2 ) \\
 = \sigma^2 \exp \left( 
- \left\{ \sum_{i=0}^{o-1} \left[ \frac{ |a_i(\nu_1)  - a_i (\nu_2)| }{ \ell_i } \right] \right\}^{H}
 \right).
\end{multline*}
The covariance parameters $\sigma^2 \geq 0, \ell_0 >0,...,\ell_{o-1} >0,H \in (0,1]$ are estimated by maximum likelihood, from the learning set $( a_0(\nu_i),...,a_{o-1}(\nu_i),F(\nu_i) )_{i=1,...,n}$. Finally, the conditional moments $\hat{F}(\nu)$ and $\hat{\sigma}^2(\nu)$ are computed as for the ``distribution'' Gaussian process model.

For the ``PCA'' covariance function, we discretize each of the $n$ probability density functions $f_{\nu_i}$ to obtain $n$ vectors $v_i = (f_{\nu_i}(j/(d-1)) )_{j=0,...,d-1}$, with $d=100$. Then, we let $w_1,...,w_o$ be the first $o$ principal component vectors of the set of vectors $(v_1,...,n_n)$. For any distribution $\nu$ with density $f_{\nu}$, we associate its projection vector $(a_1(\nu),...,a_o(\nu))$ defined as
\[
a_i(\nu) = \frac{1}{d} \sum_{j=0}^{d-1} f_{\nu}(j/(d-1))( w_i)_j.
\]

This procedure corresponds to the numerical implementation of functional principal component analysis presented in Section 2.3 of \cite{ramsay05functional}. Then, the covariance function in the ``PCA'' case operates on the input vector $(a_1(\nu),...,a_{o}(\nu))$. Finally, the conditional moments $\hat{F}(\nu)$ and $\hat{\sigma}^2(\nu)$ are computed as for the ``Legendre'' Gaussian process model.

\subsubsection{Results}

In Table \ref{table:results} we show the values of the RMSE and $CIR_{0.9}$ quality criteria for the ``distribution'', ``Legendre'' and ``PCA'' Gaussian process models. From the values of the RMSE criterion, the ``distribution'' Gaussian process model clearly outperforms the two other models. The RMSE of the ``Legendre'' and ``PCA'' models slightly decreases when the order increases, and stay well above the RMSE of the ``distribution'' model. Note that with orders $10$ and $15$, despite being less accurate, the ``Legendre'' and ``PCA'' models are significantly more complex to fit and interpret than the ``distribution'' model. Indeed these two models necessitate to estimate $12$ and $17$ covariance parameters, against $3$ for the ``distribution'' model. The maximum likelihood estimation procedure thus takes more time for the ``Legendre'' and ``PCA'' models than for the ``distribution'' model. We also remark that all three models provide appropriate predictive confidence intervals, as the value of the $CIR_{0.9}$ criterion is close to $0.9$. Finally, ``Legendre'' performs slightly better than ``PCA''.

Our interpretation for these results is that, because of the nature of the simulated data $(\nu_i,F(\nu_i))$, working directly on distributions, and with the Wasserstein distance, is more appropriate than using linear projections. Indeed, in particular, two distributions with similar means and small variances are close to each other with respect to both the Wasserstein distance and the value of the output function $F$. However, if the ratio between the two variances is large, the probability density functions of the two distributions are very different from each other, with respect to the $L^2$ distance. Hence, linear projections based on probability density functions is inappropriate in the setting considered here.

\begin{table}
\begin{center} 
\begin{tabular}{| c | c |  c |}
\hline
model & RMSE & $CIR_{0.9}$ \\ 
\hline
 ``distribution''  &  $0.094$   & $0.92$ \\
  ``Legendre'' order 5  &   $0.49$   & $0.92$ \\
``Legendre'' order 10  &  $0.34$    & $0.89$ \\
  ``Legendre'' order 15  &  $0.29$   & $0.91$  \\
    ``PCA'' order 5  &  $0.63$    & $0.82$  \\
``PCA'' order 10  &  $0.52$   & $0.87$  \\
  ``PCA'' order 15  &  $0.47$  & $0.93$ \\
\hline
\end{tabular}
\end{center} 
\caption{Values of different quality criteria for the ``distribution'', ``Legendre'' and ``PCA'' Gaussian process models. The ``distribution'' Gaussian process model is based on covariance functions operating directly on the input distributions, while ``Legendre'' and ``PCA'' are based on linear projections of the input distributions on finite-dimensional spaces. For ``Legendre'' and ``PCA'', the order value is the dimension of the projection space.
The quality criteria are the root mean square error (RMSE) which should be minimal and the confidence interval ratio ($CIR_{0.9}$) which should be close to $0.9$. The ``distribution'' Gaussian process model clearly outperforms the two other models.}
\label{table:results} 
\end{table} 

\subsection{Comparison with the kernel regression procedure of \cite{poczos13distribution}}

In this section, we compare the ``distribution'' method of Table \ref{table:results} which is suggested in the present article, with the ``kernel regression'' procedure of \cite{poczos13distribution}. This procedure consists in predicting $f(P) \in \mathbb{R}$, with $P \in \W$, from $\hat{P},\hat{P}_1,...,\hat{P}_n,f(P_1),...,f(P_n)$ where $\hat{P},\hat{P}_1,...,\hat{P}_n$ are estimates of $P,P_1,...,P_n \in \W$ obtained from sample values of $P,P_1,...,P_n$. In \cite{poczos13distribution}, $\hat{P},\hat{P}_1,...,\hat{P}_n$ correspond to kernel smoothing estimates of probability density functions constructed from the sample values. Then, the prediction $\hat{f}(\hat{P})$ of $f(P)$ is obtained by a weighted average of $f(P_1),...,f(P_n)$ where the weights are computed by applying a kernel to the distances $D(\hat{P},\hat{P}_1),...,D(\hat{P},\hat{P}_n)$. The distances suggested in \cite{poczos13distribution} are the $L^1$ distances between the estimated probability density functions. We remark that there is no estimate of the prediction error $f(P) - \hat{f}(\hat{P})$ in \cite{poczos13distribution}, which is a downside compared to the Gaussian process model considered in this paper.

An interesting feature of the setting of \cite{poczos13distribution} is that the input $P$ of the function value $f(P)$ is not observed. Only a sample from $P$ is available (this is the ``two-stage sampling" difficulty described in \cite{szabo2015two}, which arises in various applications) We shall demonstrate in this section that Gaussian process models can accommodate with this constraint. The idea is that $f(\hat{P})$ differs from $f(P)$, and that this difference can be modeled by adding a nugget variance parameter to the Gaussian process model. More precisely, the covariance functions we shall study in this section are 
\begin{multline}
K_{\sigma^2,\ell,H,\delta}( \nu_1 , \nu_2 )\\ = 
\sigma^2 \exp \left( 
- \frac{W_2( \nu_1 , \nu_2 )^{2H}}{\ell}
 \right)
 + \delta \mathbf{1} \{ W_2( \nu_1 , \nu_2 ) = 0 \} ,
\end{multline}

where $\delta \geq 0$ is an additional covariance parameter, which can also be estimated in the maximum likelihood procedure. Apart from this modification of the covariance model, we carry out the Gaussian process model computation as in Section \ref{subsection:comparison:other:cov:functions}, with always $W_2( P,Q )$ replaced by $W_2( \check{P} , \check{Q} )$, where $\check{P},\check{Q}$ are the empirical distributions corresponding to the available sample values from $P,Q$. 

We first reproduce the ``skewness of Beta'' example of \cite{poczos13distribution}. In this example $n=275$ distributions $P_1,...,P_n$ are randomly and independently generated for the learning set. We have that $P_i = B_{a_i}$ is the Beta distribution with parameters $(a_i,b)$ where $a_i$ is uniformly distributed on $[3,20]$ and $b=3$. The test set consists in $n_t = 50$ distributions $P_{t,1},...,P_{t,n_t}$ generated independently in the same way. The function to predict is defined by $f(P_a) = [2(b-a)(a+b+1)^{1/2}]/[(a+b+2)(ab)^{1/2}]$ and corresponds to the skewness of the Beta distribution. For each distribution, $500$ sample values are available. For the ``kernel regression'' procedure, we used the same settings (kernel, bandwidth selection, training and validation sets...) as in \cite{poczos13distribution}. 

The predictions obtained by the ``distribution'' and ``kernel regression'' procedures are presented in Figure \ref{fig:skewness:beta}. We observe that both methods perform equally well. The prediction errors are small, and are essentially due to to the fact that we only observe random samples from the distributions. [We have repeated the simulation of Figure \ref{fig:skewness:beta} with $5,000$ sample values instead of $500$, and the predicted values have become visually equal to the true values.] Our conclusion on this ``skewness of Beta'' example is that the setting is here very favourable (the input space of distributions is one-dimensional and $275$ observations of the function are available) so that both methods have similar good performances.

\begin{figure}
\centering
\begin{tabular}{cc}
\includegraphics[height=4cm,width=4cm]{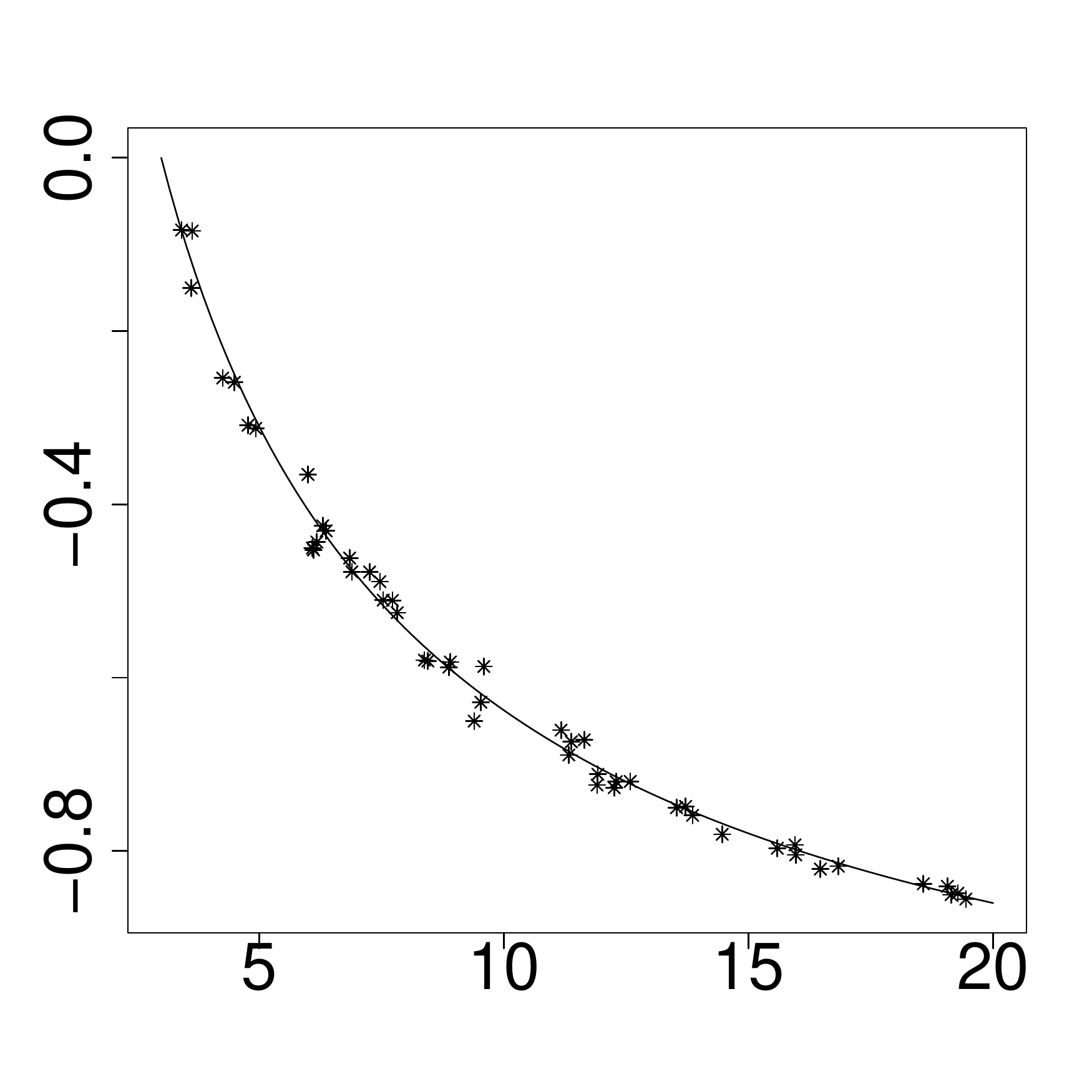}
&
\includegraphics[height=4cm,width=4cm]{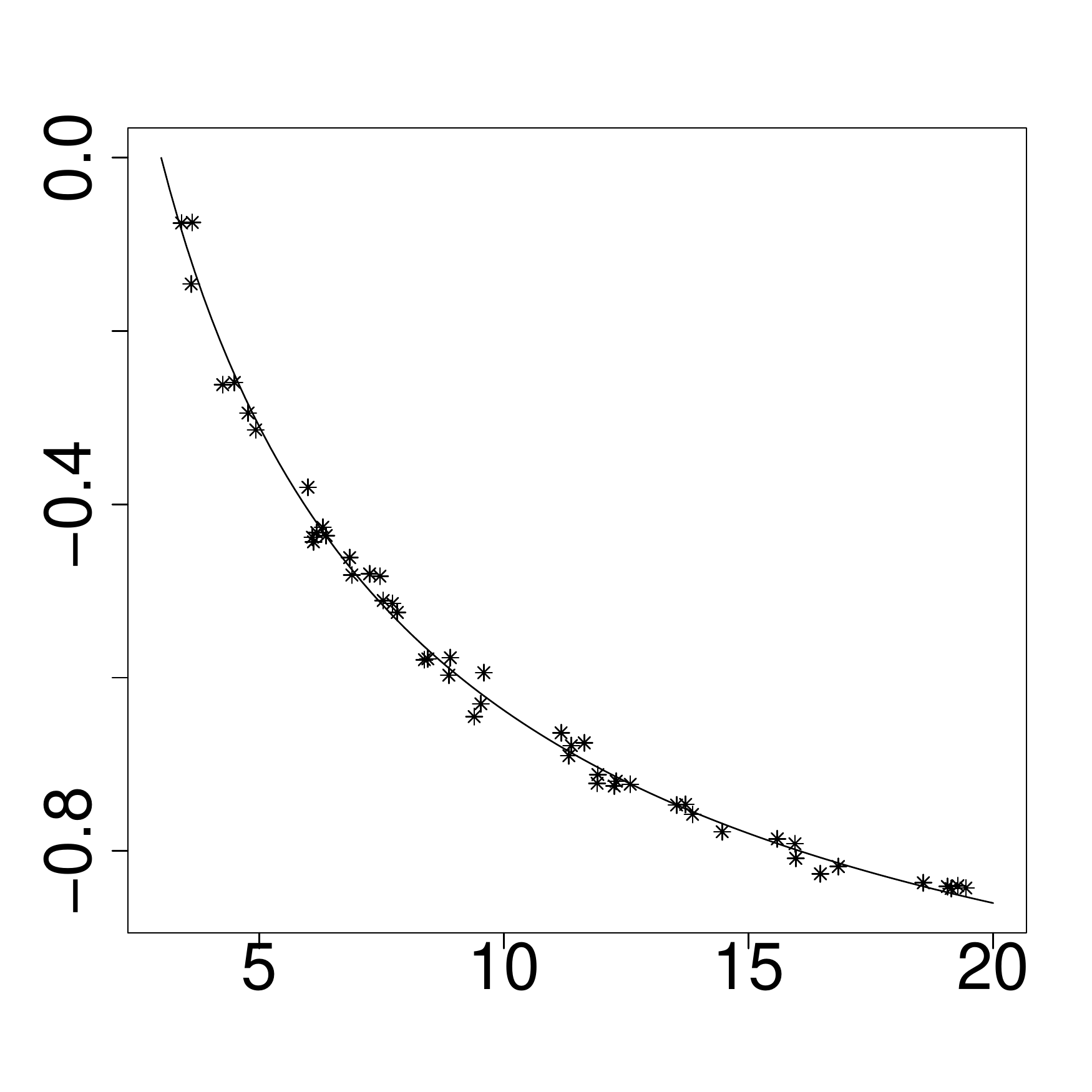}
\end{tabular}
\caption{Comparison of the ``distribution'' Gaussian process model of this paper (left) with the ``kernel regression'' procedure (right) for the ``skewness of Beta'' example. We predict the skewness of the Beta distribution (y-axis) from samples obtained from Beta distributions with parameter $(a,3)$ with $a \in [3,20]$ (x-axis). The true skewness is in plain line and the predictions are the dots. Both methods perform equally well.}
\label{fig:skewness:beta}
\end{figure}

Next, we repeat the ``distribution'' and ``kernel regression'' procedures on the same setting as in Table \ref{table:results} (except that each input and predictand distribution is only observed indirectly, through $500$ sample values from it). The prediction results, based on the same criteria as in Table \ref{table:results} are given in Table \ref{table:results:comparison:dist:reg}. We observe that the RMSE prediction criterion for the ``distribution'' model is deteriorated compared to Table \ref{table:results}. This is due to the fact that the distributions are not observed exactly anymore. The $CIR_{0.9}$ criterion is equal to $0.91$ for the ``distribution'' Gaussian process model. Hence, thanks to the addition of the nugget variance parameter, the Gaussian process model is able to take into account the additional uncertainty due to the random samples of the unobserved distributions, and to yield appropriate conditional variances.

We also observe that the RMSE pediction criterion is much larger for the ``kernel regression'' procedure. Hence, in this more challenging scenario (the input-space of distributions is non-parametric and only $100$ learning function values are available), the ``distribution'' Gaussian process model become strongly preferable. In our opinion, this is because the Wasserstein distance is here more relevant than distances between probability density functions (as discussed for Table \ref{table:results}). Also, Gaussian process prediction has benefits compared to prediction with weighted kernel averages. In particular, Gaussian process predictions come with a probabilistic model and have optimality properties under this model.

\begin{table}
\begin{center} 
\begin{tabular}{| c | c |  c |}
\hline
model & RMSE & $CIR_{0.9}$ \\ 
\hline
 ``distribution''  &  $ 0.21$   & $0.91$ \\
  ``kernel regression''  &   $0.93$   & \\
\hline
\end{tabular}
\end{center} 
\caption{Same setting as in Table \ref{table:results}, except that the input and predictand distributions are only observed indirectly, through sample values from them. The ``distribution'' model suggested in this paper clearly outperforms the ``kernel regression'' procedure.}
\label{table:results:comparison:dist:reg} 
\end{table} 

\subsection{Numerical complexity of the method}
Our method inherits the numerical complexity of Gaussian process regression in more classical settings. Given a learning dataset $(\mu_i,y_i)_{i=1}^N$ the complexity of the Kriging method is given by the inversion of the covariance matrix $K_\theta(\mu_i,\mu_j)_{i,j=1}^N$, which is in $O(N^3)$ number of operations. The Wasserstein distances between every pair of $\mu_i$ need also to be evaluated, which costs $O(N^2 q)$ operations, where $q$ is the size of the sampling of the distributions.

 Each prediction is then obtained by a vector product in $O(N)$ operations, while the computation of the conditional variance at some outputs is obtained in $O(N^2)$.
 
 The $O(N^3)$ cost of the overall method makes it challenging to use on very large datasets, however on moderately large datasets its good performances makes it an interesting choice, and in particular a preferable choice over the other methods it was compared to in this simulation study.
 
 For the sake of illustration, we remark that it took around 9 seconds to carry out our whole suggested Gaussian process procedure, in the case of Table I, and around 30 seconds in the case of Table II.
 
  See also \cite{furrer2006covariance} for a discussion of the covariance tapering method to reduce the numerical cost of Gaussian process regression.

\section{Conclusion}

We provided a new approach to learning with distribution inputs. Its strength relies on the existence of positive definite kernels on the distribution space, which enables the use of Gaussian process models and kernel learning methods. In particular, we generalized the seminal models that are the fractional Brownian motion and the power exponential stationary processes, to distribution inputs. The kernels we use are functions of the Wasserstein distance, which has proven its efficiency as a discrepancy measure between distributions in numerous applications. Our method requires only the distributions inputs to have a second order moment, which allows the simultaneous handling of very heterogeneous data, such as absolutely continuous distributions, deterministic inputs and empirical distributions, which is particularly important when only a sample of the input distributions is known.

Focusing on Gaussian process regression with stationary covariance functions, we proved that our method extends this classical tool to distribution inputs. In particular, we gave generalization of state of the art asymptotic results to our setting. As in vector input Kriging, the overall numerical complexity of the method is in $O(n^3)$, where $n$ is the size of the dataset, which is more costly than other distribution regression methods (such as the kernel regression procedure from \cite{poczos13distribution}), however our numerical simulations suggest that our method gives better prediction. Furthermore Kriging comes with an error estimation in the form of the conditional variance of the Gaussian process, which is an important guarantee in practice.

On the down side, the methods we use to prove the positive definiteness of our kernels are tightly related to the existence of an optimal coupling between every distribution, which existence is specific to dimension one. It is an important problem for numerous applications to give learning methods for multidimensional distributions. Hence, it would be valuable to obtain kernels based on the multidimensional Wasserstein space. This would require an other approach that the one used in the present paper, and constitutes an interesting problem for further research.

\appendix[Proofs] \label{section:proofs}

\subsection{Proofs for Section \ref{section:kernels}} \label{subsec:proofs_kernels}
\begin{proof}[Proof of Theorem \ref{thm:negative_definite}]
	
	We start with the negative definiteness.  For any $\mu \in \mathcal{W}_2(\mathbb{R})$ we denote by $F_\mu^{-1}$ the quantile function associated to $\mu$. It is well known that given a uniform random variable $U$ on $[0,1]$, $F_\mu^{-1}(U)$ is a random variable with law $\mu$, and furthermore for every $\mu,\nu \in \mathcal{W}_2(\mathbb{R})$:
	\begin{equation} \label{eq:optimal} W_2^2(\mu,\nu)=\E\left(F_\mu^{-1}(U)-F_\nu^{-1}(U)\right)^2,\end{equation}
	that is to say the coupling of $\mu$ and $\nu$ given by the random vector $(F_\mu^{-1}(U),F_\nu^{-1}(U))$ is optimal. Consider now $\mu_1,\cdots,\mu_n \in \mathcal{W}_2(\mathbb{R})$ and $c_1,\cdots,c_n \in \mathbb{R}$ such that $\sum_{i=1}^n c_i=0$.
	We have
	\begin{align*} & \sum_{i,j=1}^{n} c_i c_j W_2^2(\mu_i,\mu_j)  \\   = & \sum_{i,j=1}^{n} c_i c_j\E \left(F_{\mu_i}^{-1}(U)-F_{\mu_j}^{-1}(U)\right)^2 \\
	=  & \sum_{i,j=1}^{n} c_i c_j \E \left(F_{\mu_i}^{-1}(U)\right)^2 + \sum_{i,j=1}^{n} c_i c_j  \E \left(F_{\mu_j}^{-1}(U)\right)^2\\ &  -2 \sum_{i,j=1}^{n} c_i c_j  \E \left(F_{\mu_i}^{-1}(U) F_{\mu_j}^{-1}(U)\right).
	\end{align*}
	Using $\sum_{i=1}^n c_i=0$ the first two sums vanish and we obtain  \begin{align*} & \sum_{i,j=1}^{n} c_i c_j W_2^2(\mu_i,\mu_j)  \\ = & -2 \sum_{i,j=1}^{n} c_i c_j  \E \left(F_{\mu_i}^{-1}(U) F_{\mu_j}^{-1}(U)\right)\\
	=& -2 \E \left( \sum_{i=1}^n c_i F^{-1}_{\mu_i}(U)\right)^2 \leq 0,\end{align*} which proves that $W^{2H}_2$ is a negative definite kernel for $0\leq H \leq 1$.
	
	Let us now consider $H>1$. Using \eqref{eq:quadratic_cost} it is clear that for every $x,y\in \mathbb{R}$, $W_2(\delta_x,\delta_y)=|x-y|$. It is well known (see \textit{e.g} \cite{istas2011manifold}) that $|x-y]^{2H}$ is not a negative definite kernel on $\mathbb{R}$ for $H>1$, hence the same is true for $W_2^{2H}$.

	Let us now prove the nondegeneracy of the kernel: the idea of the proof is adapted from \cite{Venet_critical}: we consider $\mathcal{W}_2(\mathbb{R}) \times \mathbb{R}$ endowed with the product distance
	$$d((\mu,s),(\nu,t))=\left(W_2(\mu,\nu)^2+|s-t|^2\right)^{1/2}.$$ We assume the degeneracy of the kernel $W^{2H}_2$ on $\W$ and deduce that $d^{2H}$ is not negative definite on $\mathcal{W}_2(\mathbb{R})\times \mathbb{R}$, in contradiction with the following Lemma, from which we postpone the proof:
	\begin{lem} \label{lem:product_index} The function $d^{2H}$ is a negative definite kernel if and only if $0\leq H\leq 1$.
	\end{lem}
	
	Let us fix $0<H<1$ and assume that $W_2^{2H}$ is degenerate. There exists $\mu_1,\cdots,\mu_n \in \W $ and $c_1,\cdots,c_n \in \mathbb{R}$ such that $\sum_{i=1}^n c_i =0$ and
	
	\begin{equation} \label{eq:degenerate} \sum_{i,j=1}^n c_i c_j W_2^{2H}(\mu_i,\mu_j)=0.
	\end{equation}
	In $\W \times \mathbb{R}$ we now consider the points $P_i=(\mu_i,0)$ for $1\leq i \leq n$ and $P_{n+1}=(\mu_n,\varepsilon)$ with $\varepsilon > 0$. We also set $c'_i=c_i$ for every $1 \leq i \leq n-1$ and $c'_n=c'_{n+1}=c_n/2$. Notice that we have
	$$\sum_{i=1}^{n+1} c'_i=0. $$ Now
	\begin{align*}& \sum_{i,j=1}^{n+1}c'_i c'_j d^{2H}(P_i,P_j) \\  = & \sum_{i,j=1}^{n-1}c'_i c'_j d^{2H}(P_i,P_j)+2\sum_{i=1}^{n-1}c'_i c'_n d^{2H}(P_i,P_n) \\ &+2\sum_{i=1}^{n-1}c'_i c'_{n+1} d^{2H}(P_i,P_{n+1})+2c'_{n} c'_{n+1} d^{2H}(P_n,P_{n+1}).\end{align*}
	
	We now use \begin{align*} d^{2H}(P_i,P_{n+1})= & \left(W_2(\mu_i,\mu_n)^2 + \varepsilon^2 \right)^{H} \\ = & W_2(\mu_i,\mu_n)^{2H}+O\left(\varepsilon^{2}\right)\end{align*}
	to obtain
	\begin{align*}
	& \sum_{i,j=1}^{n+1}c'_i c'_j d^{2H}(P_i,P_j) \\  = & \sum_{i,j=1}^{n-1}c_i c_j W_2^{2H}(\mu_i,\mu_j)+2\sum_{i=1}^{n-1}c_i \frac{c_n}{2} W_2^{2H}(\mu_i,\mu_n) \\ &+2\sum_{i=1}^{n-1}c_i \frac{c_{n}}{2}  W_2^{2H}(\mu_i,\mu_n)+\frac{c_n^2}{2}\varepsilon^{2H}+O\left(\varepsilon^2\right) \\
	= & \sum_{i,j=1}^{n-1}c_i c_j W_2^{2H}(\mu_i,\mu_j)+2\sum_{i=1}^{n-1}c_i c_n W_2^{2H}(\mu_i,\mu_n) \\ &+\frac{c_n^2}{2}\varepsilon^{2H}+O\left(\varepsilon^2\right) \\
	= & \sum_{i,j=1}^n c_i c_j W_2^{2H}(\mu_i,\mu_j) +\frac{c_n^2}{2}\varepsilon^{2H}+O\left(\varepsilon^{2}\right).
	\end{align*}
	Finally using \eqref{eq:degenerate} and $H<1$ we obtain $$\sum_{i,j=1}^{n+1}c'_i c'_j d^{2H}(P_i,P_j)=\frac{c_n^2}{2}\varepsilon^{2H}+o \left(\varepsilon^{2H}\right),$$ which is positive for $\varepsilon$ small enough. This shows that $d^{2H}$ is not negative definite, in contradiction with Lemma \ref{lem:product_index}. In the end $W_2^{2H}$ is nondegenerate for every $0 < H < 1$.
	
	We now use the same argument as in the end of the proof of Theorem \ref{thm:negative_definite}. Since $W_2^{2H}(\delta_x,\delta_y)=|x-y|^{2H}$ and $|x-y|^2$ and $|x-y|^0$ are degenerate kernels on $\mathbb{R}$, $W_2^0$ and $W_2^2$ are degenerate kernels.
	
\end{proof}
\begin{proof}[Proof of Lemma \ref{lem:product_index}] For $H=1$ we have
	$$d^2((\mu,s),(\nu,t))=W_2(\mu,\nu)^2+|s-t|^2 $$
	hence $d^2$ is negative definite as the sum of two negative definite kernels. From Lemma \ref{lem:subordination} we get that $d^{2H}$ is a negative definite kernel for every $0\leq H \leq 1$.
	
	For $H>1$ we notice that $d^{2H}(\mu,x)(\mu,y)=|x-y|^{2H}$ and use again the fact that $|x-y|^{2H}$ is not a negative definite kernel to conclude that $d^{2H}$ is not negative definite.
\end{proof}

\begin{proof}[Proof of Theorem \ref{thm:fractional_brownian_kernels}]
	The fact that \eqref{eq:fbm_kernels} are covariance kernels is a direct consequence of Theorem \ref{thm:negative_definite} and the following Schoenberg Theorem (which is proven in \cite{berg_al}):
	\begin{theorem}[Schoenberg] \label{thm:Schoenberg_definite} Given a set $X$, two functions $K,R: X\times X \rightarrow \mathbb{R}$, and $o\in X$ such that for every $x,y \in X$,
		$$K(x,x)=0$$ and
		$$R(x,y)=K(x,o)+K(y,o)-K(x,y),$$
		the function $R$ is a positive definite kernel if and only if $K$ is a negative definite kernel.
	\end{theorem}	
	
	We now prove the degeneracy: let $X=(X(\mu))_{\mu \in\W}$ denote the $H$-fractional Brownian field indexed by $\mathcal{W}_2(\mathbb{R})$ with origin in $\sigma$.  Assume $X$ is degenerate: there exist $\lambda_1,\cdots,\lambda_n \in \mathbb{R}$ and $\mu_1,\cdots,\mu_n \in \W$ such that
	$$\sum_{i=1}^{n} \lambda_n X(\mu_n)= 0 \text{ almost surely.}$$
	Since $X(\sigma)=0$ almost surely, setting $\mu_{n+1}=\sigma$ and $\lambda_{n+1}=-\sum_{i=1}^n \lambda_i$, it is clear that
	$$\sum_{i=1}^{n+1} \lambda_n X(\mu_n)= 0 \text{ almost surely,}$$
	which implies $$\sum_{i,j=1}^{n+1}\lambda_i \lambda_j W_2^{2H}(\mu_i,\mu_j)=\E\left(\sum_{i=1}^{n+1} \lambda_n X(\mu_n)\right)^2=0.$$
	Since $\sum_{i=1}^{n+1} \lambda_i = 0$ this shows that $W_2^{2H}$ is degenerate, in contradiction with Theorem \ref{thm:negative_definite}. Therefore $X$ is nondegenerate for every $0<H<1$.
	
	The degeneracy of the $0$-fractional and the $2$-fractional Brownian field indexed by $\W$ is a direct consequence from the degeneracy of $W_2^0$ and $W_2^2$. 

\end{proof}

\begin{proof}[Proof of Theorem \ref{thm:stationary_valid_kernels}] The fact that \eqref{eq:stationnary_covariance} are covariance kernels is a direct consequence of Theorem \ref{thm:negative_definite} and the following Schoenberg Theorem (which proof can be found in \cite{berg_al}):
	
	\begin{theorem}[Schoenberg]\label{thm:Schoenberg_monotone} Let $F:\mathbb{R}^+ \rightarrow \mathbb{R}^+$ be a completely monotone function, and $K$ a negative definite kernel. Then $(x,y)\mapsto F(K(x,y))$ is a positive definite kernel.
	\end{theorem}
	
	 Furthermore as a function of the distance $W_2$, \eqref{eq:stationnary_covariance} is obviously invariant under the action of any isometry of $\W$, so that the second claim holds.
\end{proof}

\subsection{Proofs for Section \ref{subsection:asymptotic:properties}}

\begin{proof}[Proof of Theorem \ref{theorem:consistency}] We have $\ML \in \argmin L_\theta$ with
	$$L_\theta= \frac{1}{n} \ln (\det R_\theta) +\frac{1}{n}y^t R_\theta^{-1} y. $$
	From Lemma \ref{lem:largest_eigenvalue} we have that
	$$\displaystyle \sup\limits_{\theta \in \Theta} \lambda_{\max}(R_\theta) \text{~~ and ~~ }  \displaystyle\sup\limits_{\theta \in \Theta} \max\limits_{i=1,\cdots,p} \lambda_{\max}\left(\frac{\partial}{\partial \theta_i}R_\theta \right)$$ are bounded as $n \rightarrow \infty$.
	Hence we can proceed as in the beginning of the proof of Proposition 3.1 in \cite{bachoc14asymptotic} to obtain
	\begin{equation} \label{eq:for:ML:cons:lik}
	\st \| L_\theta - \E(L_\theta) \|=o_{\mathbb{P}}(1). 
	\end{equation}
	Following again the proof of Proposition 3.1 in \cite{bachoc14asymptotic} we obtain the existence of a positive $a$ such that
	$$\E(L_\theta)-\E(L_{\theta_0}) \geq a | R_\theta - R_{\theta_0} |^2 ,$$
	with $| \Lambda |^2 = (1/n) \sum_{i,j=1}^n \Lambda_{i,j}^2$.
	
	Hence from Condition \ref{cond:asymptotics:cinq} and \eqref{eq:for:ML:cons:lik} we have $\forall \alpha >0$, $$\bbP \left(\left\| \ML - \theta_0 \right\| \geq \alpha \right) \underset{n\rightarrow \infty}{\longrightarrow} 0$$ and so  $$\ML {\overset{\bbP}{\underset{n \rightarrow \infty}{\longrightarrow}}} \theta_0. $$
\end{proof}

\begin{proof}[Proof of Theorem \ref{theorem:TCL}] From Lemma \ref{lem:largest_eigenvalue} and Condition \ref{cond:asymptotics:quatre} we have for every $n \in \mathbb{N}$, $\left| \left(M_{ML}\right)_{i,j} \right| \leq B$ for a fixed $B < \infty$.
	
	In addition, for any $\lambda_1,\cdots, \lambda_p \in \mathbb{R}$ such that $\sum_{i=1}^p \lambda_i^2 =1 $,
	\begin{align*}
	& \sum_{i,j=1}^p \lambda_i \lambda_j \left(M_{ML} \right)_{i,j} \\ & =  \frac{1}{2n} \Tr \left( R_{\theta_0}^{-1} \left(\sum_{i=1}^p \lambda_i \frac{\partial R_{\theta_0}}{\theta_i}\right) R_{\theta_0}^{-1} \left(\sum_{j=1}^p \lambda_j \frac{\partial R_{\theta_0}}{\theta_j}\right) \right) \\
	& =  \frac{1}{2} \left| R_{\theta_0}^{-1/2} \left( \sum_{i=1}^p  \lambda_i \frac{\partial R_{\theta_0}}{\partial \theta_i }\right) R_{\theta_0}^{-1/2} \right|^2 \\
	& \geq  C^2 \left| \sum_{i=1}^p \lambda_i \frac{\partial R_{\theta_0}}{\partial \theta_i} \right|^2 
	\end{align*}
	with a fixed $C>0$, since for every $n$
	$$\lambda_{\min}\left(R_{\theta_0}^{-1} \right)= \frac{1}{\lambda_{\max}(R_{\theta_0})} \geq C >0$$ from Lemma \ref{lem:largest_eigenvalue}. Hence from Condition \ref{cond:asymptotics:huit} we obtain
	$$ \liminf\limits_{n \rightarrow \infty} \lambda_{\min} (M_{ML}) >0. $$
	
	Hence \eqref{eq:asymptotic:cov:mat:bounded} is proved.
	Let us now assume that
	\begin{equation} \label{eq:hyp_abs}  \sqrt{n} M_{ML}^{1/2} \left( \ML - \theta_0\right) \cancel{{\overset{\mathcal{L}}{\underset{n \rightarrow \infty}{\longrightarrow}}}} \mathcal{N}(0,I_n).\end{equation}
	
	Then there exists a  bounded measurable function $g : \mathbb{R}^p \rightarrow \mathbb{R}$, $\xi>0$ and a subsequence $n'$ such that along $n'$ we have
	
	$$\left| \E \left[ g\left( \sqrt{n} M_{ML}^{1/2} (\ML - \theta_0 )\right)\right]-\E(g(U)) \right| \geq \xi, $$
	with $U \sim \mathcal{N}(0,I_p)$.
	
	In addition, by compactness, up to extracting another subsequence we can assume that
	$$M_{ML} \underset{n\rightarrow \infty}{\rightarrow} M_\infty, $$
	where $M_\infty$ is a symmetric positive definite matrix.
	
	Now the remaining of the proof is similar to the proof of Proposition 3.2 in \cite{bachoc14asymptotic}. We have
	$$\frac{\partial}{\partial \theta_i} L_\theta= \frac{1}{n} \left( \Tr \left(R_\theta^{-1} \frac{\partial R_\theta}{\partial\theta_i} \right)-y^t R_{\theta}^{-1} \frac{\partial R_{\theta}}{\partial \theta_i}R_{\theta}^{-1} y \right).$$
	Hence, exactly as in the proof of Proposition D.9 in \cite{bachoc14asymptotic} we can show
	$$\sqrt{n} \frac{\partial}{ \partial \theta_i} L_{\theta_0} {\overset{\mathcal{L}}{\underset{n \rightarrow \infty}{\longrightarrow}}} \mathcal{N}(0, 4 M_\infty).  $$
	
	Let us compute
	
	\begin{align*}
	&\frac{\partial^2}{\partial\theta_i \partial \theta_j} L_{\theta_0}= \frac{1}{n} \Tr \left(-R_{\theta_0}^{-1}\frac{\partial R_{\theta_0}}{\partial\theta_i} R_{\theta_0}^{-1}\frac{R_{\theta_0}}{\partial\theta_j}+R_{\theta_0}^{-1} \frac{\partial^2 R_{\theta_0}}{\partial\theta_i \partial \theta_j} \right) \\
	& + \frac{1}{n} y^t \left( 2 R_{\theta_0}^{-1} \frac{\partial R_{\theta_0}}{\partial\theta_i} R_{\theta_0}^{-1} \frac{\partial R_{\theta_0}}{\partial \theta_j} R_{\theta_0}^{-1}-R_{\theta_0}^{-1}\frac{\partial^2 R_{\theta_0}}{\partial \theta_i \partial \theta_j} R_{\theta_0}^{-1} \right) y. 
	\end{align*}
	We have $$\E \left( \frac{\partial^2}{\partial \theta_i \partial \theta_j} L_{\theta_0}\right)=2 M_{ML}, $$
	and from Condition \ref{cond:asymptotics:quatre} and Lemma \ref{lem:largest_eigenvalue_bis}, $$\Var\left(\frac{\partial^2}{\partial \theta_i \partial \theta_j} L_{\theta_0} \right) \underset{n\rightarrow \infty}{\longrightarrow} 0. $$
	
	Hence $$\frac{\partial^2}{\partial \theta_i \partial \theta_j} L_{\theta_0} {\overset{\mathbb{P}}{\underset{n \rightarrow \infty}{\longrightarrow}}} 2 M_\infty. $$
	
	Moreover, $\displaystyle \frac{\partial^3}{\partial \theta_i \partial \theta_j \partial \theta_k} L_{\theta}$ can be written as
	$$\frac{1}{n} \Tr (A_\theta)+\frac{1}{n} y^t B_\theta y, $$
	where $A_\theta$ and $B_\theta$ are sums of products of the matrices $R_\theta^{-1}$ or $\displaystyle \frac{\partial}{\partial \theta_{i_1}}\cdots \frac{\partial}{\partial \theta_{i_q}} R_\theta$ with $q\in \{0,\cdots,3\}$ and $i_1,\cdots,i_q \in \{1,\cdots p \}$.
	
	Hence from Condition \ref{cond:asymptotics:quatre} and from Lemmas \ref{lem:largest_eigenvalue} and \ref{lem:largest_eigenvalue_bis} we have
	$$ \st \left\| \frac{\partial^3}{\partial \theta_i \partial \theta_j \partial \theta_j}L_\theta  \right\| = O_{\bbP}(1). $$
	
	Following exactly the proof of Proposition D.10 in \cite{bachoc14asymptotic} we can show that
	$$\sqrt{n} (\ML - \theta_0) {\overset{\mathcal{L}}{\underset{n' \rightarrow \infty}{\longrightarrow}}} \mathcal{N}(0,M_\infty^{-1}).  $$
	Moreover since $M_{ML}\underset{n\rightarrow \infty}{\rightarrow} M_\infty$ we have
	
	$$ \sqrt{n} M_{ML}^{1/2} (\ML-\theta_0)  {\overset{\mathcal{L}}{\underset{n' \rightarrow \infty}{\longrightarrow}}} \mathcal{N}(0,I_p). $$
	
	This is in contradiction with \eqref{eq:hyp_abs} and concludes the proof.
\end{proof}

\begin{proof}[Proof of Theorem \ref{theorem:input_prediction}] From Theorem \ref{theorem:consistency} it is enough to show for $i=1,\cdots,p$ that
	$$\st \left| \frac{\partial}{\partial \theta_i} \hat{Y}_\theta(\mu)\right|=O_{\bbP}(1).$$
	From a version of Sobolev embedding theorem (see Theorem 4.12, part I, case A in \cite{adams03sobolev}), there exists a finite constant $A_\Theta$ depending only on $\Theta$ such that
	\begin{align*}\st \left| \frac{\partial}{\partial \theta_i} \hat{Y}_\theta (\mu)\right| & \leq A_\Theta \int_\Theta \left| \frac{\partial}{\partial \theta_i} \hat{Y}_\theta (\mu) \right|^{p+1} d\theta \\
	& + A_\Theta \sum_{j=1}^p \int_\Theta \left| \frac{\partial}{\partial \theta_j} \frac{\partial}{\partial \theta_i} \hat{Y}_\theta(\mu) \right|^{p+1} d\theta .\end{align*}
	Therefore in order to prove the theorem it is sufficient to show that for $w_\theta(\mu)$ of the form $r_\theta(\mu)$ or $\frac{\partial}{\partial \theta_i} r_\theta(\mu)$ or $\frac{\partial}{\partial \theta_i}\frac{\partial}{\partial \theta_j}r_\theta(\mu)$, and for $W_\theta$ equal to a product of the matrices $R_\theta^{-1}$ or $\frac{\partial}{\partial \theta_i} R_\theta$ or $\frac{\partial}{\partial \theta_i}\frac{\partial}{\partial \theta_j} R_\theta$, we have
	$$\int_\Theta \left| w_\theta^t(\mu) W_\theta y\right |^{p+1} d\theta=O_{\bbP}(1). $$
	
	From Fubini theorem for positive integrands we have
	
	$$\E \left[ \int_\Theta \left| w_\theta^t (\mu) W_\theta y \right|^{p+1} d\theta \right] = \int_\Theta \E \left( \left| w_\theta^t(\mu) W_\theta y \right|^{p+1}\right) d\theta.$$
	
	Now there exists a constant $c_{p+1}$ so that for $X$ a centred Gaussian random variable,
	$$ \E \left( |X|^{p+1}\right)=c_{p+1} \left(\Var(X)\right)^{(p+1)/2},$$ hence
	
	\begin{align*}
	& \E \left( \int_\Theta \left| w_\theta^t (\mu) W_\theta y \right|^{p+1} d \theta \right)\\
	& =  c_{p+1} \int_\Theta \left(\Var \left(w_\theta^t(\mu) W_\theta y\right) \right)^{(p+1)/2} d\theta\\
	& =  c_{p+1} \int_\Theta \left( w_\theta^t(\mu) W_\theta R_{\theta_0} W_\theta^t w_\theta(\mu)\right)^{(p+1)/2} d\theta.
	\end{align*}
	
	Now from Lemmas \ref{lem:largest_eigenvalue} and \ref{lem:largest_eigenvalue_bis} there exists $B< \infty$ such that
	$$ \st \lambda_{\max} \left( W_\theta R_{\theta_0} W_\theta\right) \leq B.$$
	
	Thus
	
	\begin{align*}
	\E \left( \int_\Theta \left| w_\theta^t W_\theta y \right|^{p+1} d\theta \right)\leq B^{(p+1)/2} c_{p+1} \int_\Theta \left\| w_\theta^t(\mu)\right\|^{(p+1)/2} d\theta.
	\end{align*}
	
	Finally for some $q\in \{ 0,1,2\}$ and for $i_1,\cdots, i_q \in \{1,\cdots p \}$ we have
	
	\begin{align*}
	\st \left\| w_\theta^t (\mu) \right\|^2  & = \st \sum_{i=1}^n \left( \frac{\partial}{\partial \theta_{i_1}}\cdots \frac{\partial}{\partial \theta_{i_q}} F_{\theta}(W_2(\mu,\mu_i)) \right)^2\\
	& \leq C \sum_{i=1}^n \left| \frac{1}{1+W_2(\mu,\mu_i)^{1+\tau}}\right|,
	\end{align*}
	with $C<\infty$ coming from Condition \ref{cond:asymptotics:deux}, \ref{cond:asymptotics:six}, and \ref{cond:asymptotics:sept}.
	
	Using the proof of Lemma \ref{lem:row_sum} we see that this quantity is bounded, which finishes the proof of Theorem \ref{theorem:input_prediction}.
\end{proof}

\subsection{Technical lemmas for Section \ref{subsection:asymptotic:properties}}

\begin{lem} \label{lem:row_sum} $$\sup\limits_{\mu \in W_2(\mathbb{R})} \st \sum_{j=1}^n | K_\theta(\mu,\mu_j) | $$
	is bounded as $n \rightarrow \infty$.
\end{lem}
\begin{proof}
	Let $\mu \in \W$ and $i^* \in \argmin\limits_{k\in \{ 1,\cdots n\}} W_2(\mu_k,\mu)$.
	For every $j \in \{1,\cdots, n\},$ $W_2(\mu,\mu_j) \geq W_2(\mu, \mu_{i^*}).$ Moreover from the triangle inequality we have
	$$W_2(\mu,\mu_j) \geq W_2(\mu_j,\mu_{i^*})-W_2(\mu_{i^*},\mu), $$
	
	hence $$ W_2(\mu,\mu_j) \geq \frac{W_2(\mu_j,\mu_{i^*})}{2}. $$
	Let us define $$r_\mu := \st \sum_{i=1}^n F_\theta(W_2(\mu_i,\mu) )$$
	From Condition \ref{cond:asymptotics:deux} we have \begin{align*} r_\mu & \leq \sum_{i=1}^n \frac{A}{1+W_2(\mu_i,\mu)^{1+\tau}}  \leq \sum_{i=1}^n \frac{A}{1+\left(\frac{W_2(\mu_j,\mu_{i^*})}{2}\right)^{1+\tau}}. \end{align*}
	
	Now $$W_2^2(\mu_j,\mu_{i^*})=\int_0^1 \left| q_{\mu_j}(t)-q_{\mu_{i^*}}(t) \right|^2 dt,$$
	where for every $t\in [0,1]$ $$q_\mu(t)=\inf \{ x \in \bbR | \ F_{\mu}(x) \geq t \}.$$
	
	Note that from Condition \ref{cond:asymptotics:un} for every $t \in [0,1]$, $$q_{\mu_i}(t) \in [i,i+L].$$
	
	If  $|j-i^*|\geq L$ we have
	$$\forall t \in \bbR, \ |q_{\mu_{i^*}}(t) - q_{\mu_j}(t)| \geq |j-i^* | - L$$ 
	so that 
\begin{equation} \label{eq:lowe:bound:distij}
	W_2(\mu_{i^*},\mu_j) \geq |j-i^*|-L. 
\end{equation}
	
	Hence \begin{align*} r_\mu & \leq 2 A L + \sum\limits_{j,\ |j-i^*| \geq L} \frac{A}{1+\left(\frac{|j-i^*|-L}{2}\right)^{1+\tau}} \\
	& \leq 2 A L + \sum_{j=-\infty}^{+\infty} \frac{A}{1+\left|\frac{j}{2} \right|^{1+\tau}}  < \infty.
	\end{align*}
\end{proof}

\begin{lem} \label{lem:largest_eigenvalue} Under Conditions \ref{cond:asymptotics:un} to \ref{cond:asymptotics:quatre},
	$$ \st \lambda_{\max} (R_\theta) $$ and $$\st \mi \lambda_{\max} \left( \frac{\partial}{\partial \theta_i} R_\theta \right)$$ are bounded as $n \rightarrow \infty$.
\end{lem}
\begin{proof} $$\st \lambda_{\max}(R_\theta) \leq \st \max_{i=1,...,n} \sum_{j=1}^n \left| F_\theta (W_2(\mu_i,\mu_j))\right| $$ is bounded as $n\rightarrow \infty$ from Lemma \ref{lem:row_sum}. The proof is similar for $$\st \mi \lambda_{\max} \left( \frac{\partial}{\partial \theta_i} R_\theta \right).$$
\end{proof}
In a similar way we also obtain the following Lemma.

\begin{lem} \label{lem:largest_eigenvalue_bis}
	$\forall q \in \{2,3\}$, $\forall i_1,\cdots,i_q \in \{1,\cdots p\}$,
	$$\st \lambda_{\max} \left( \frac{\partial}{\partial \theta_{i_1}}\cdots\frac{\partial}{\partial \theta_{i_q}} R_\theta \right)$$ is bounded as $n \rightarrow \infty$.
\end{lem}

\subsection{Proofs for Section \ref{subsection:asymptotics:example}}

\begin{prop} \label{prop:for:proof:example}
Under the setting of Proposition \ref{prop:asymptotics:example}, almost surely as $n \to \infty$,
\begin{multline*}
\sup_{\theta \in \Theta}
\left|
\frac{1}{n} \sum_{i,j=1}^n
\left[
 K_{\theta}(\mu_i,\mu_j)
-
K_{\theta_0}(\mu_i,\mu_j)
 \right]^2
\right.  \\
 \left .
- \sum_{j=-\infty}^{\infty}
\mathbb{E}
\left(
\left[
 K_{\theta}(\mu_0,\mu_j)
-
K_{\theta_0}(\mu_0,\mu_j)
 \right]^2
\right)
\right| \to 0
\end{multline*}
and the sum in the right-hand side of the above display is a continuous function of $\theta$.
\end{prop}

\begin{proof}
Let
\[
S_{\theta} = 
\frac{1}{n} \sum_{i,j=1}^n
\left[
 K_{\theta}(\mu_i,\mu_j)
-
K_{\theta_0}(\mu_i,\mu_j)
 \right]^2.
\]
Let $(m_n)_{n \in \mathbb{N}}$ be a sequence of integers so that as $n \to \infty$, $m_n \to \infty$ and $n/m_n \to \infty$. Let
\[
S_{\theta,m_n} = 
\frac{1}{n} \sum_{i,j=1}^n
\mathbf{1}_{ \left\{
\lfloor \frac{i-1}{m_n} \rfloor = \lfloor \frac{j-1}{m_n} \rfloor
\right\} }
\left[
 K_{\theta}(\mu_i,\mu_j)
-
K_{\theta_0}(\mu_i,\mu_j)
 \right]^2.
\]
With the same proof as that of Lemma D.11 in \cite{bachoc14asymptotic}, we can show (using \eqref{eq:lowe:bound:distij}) that $| S_{\theta} - S_{\theta,m_n} |$ goes almost surely to zero as $n \to \infty$. Also
\begin{flalign*}
S_{\theta,m_n}
& =
\frac{1}{n/m_n}
\sum_{k=0}^{ \lfloor \frac{n}{m_n} \rfloor -1}
\frac{1}{m_n}
\sum_{i,j=1}^{m_n}
\left[
 K_{\theta}(\mu_{k m_n + i},\mu_{k m_n + j})
 \right.
 & \\
 &
 \left.
-
K_{\theta_0}(\mu_{k m_n + i},\mu_{k m_n + j})
 \right]^2
& \\
& + \frac{1}{n}    \! \! \! \!  \! \!  \! \!  \! \!  \! \!  \sum_{\substack{i,j=\\m_n \left( \lfloor \frac{n}{m_n} \rfloor - 1 \right) +1}}^n \! \! \! \! \! \! \! \! \! \! \! \! 
\mathbf{1}_{ \left\{
\lfloor \frac{i-1}{m_n} \rfloor = \lfloor \frac{j-1}{m_n} \rfloor
\right\} }
\left[
 K_{\theta}(\mu_i,\mu_j)
-
K_{\theta_0}(\mu_i,\mu_j)
 \right]^2
 & \\
 & =
 \frac{1}{n / m_n} 
  \sum_{k=0}^{ \lfloor \frac{n}{m_n} \rfloor -1}
  B_k
  +r,
\end{flalign*}
say. From \eqref{eq:lowe:bound:distij}, one can show simply that $r \to 0$ almost surely as $n \to \infty$.
Also, the $B_k$ are independent random variables with identical distribution, and they are bounded in absolute value by
\[
2L+1 + \sum_{i=-\infty}^{\infty}
2 \left(
\frac{A}{1+|i|^{1+\tau}}
\right)^2
< \infty
\]
from \eqref{eq:lowe:bound:distij} and Condition \ref{cond:asymptotics:deux}. Hence, applying Theorem 2.1 in \cite{hu97strong} yields
\[
\left(
 \frac{1}{n / m_n} 
  \sum_{k=0}^{ \lfloor \frac{n}{m_n} \rfloor -1}
  B_k
\right)
  -
  \mathbb{E}(B_0)
  \to^{a.s.}_{n \to \infty} 0.
\]
Hence, finally he have obtained almost surely as $n \to \infty$
\[
\left|
S_{\theta}
-
\frac{1}{m_n}
\sum_{i,j=1}^{m_n}
\mathbb{E}
\left[
\left(
K_{\theta}(\mu_i,\mu_j)
-
K_{\theta_0}(\mu_i,\mu_j)
\right)^2
\right]
\right|
\to 0.
\]

Also, we have, for $|i-j| \geq L$
\begin{alignat*}{1}
& \mathbb{E} 
\left[
\left(
K_{\theta}(\mu_i,\mu_j)
-
K_{\theta_0}(\mu_i,\mu_j)
\right)^2
\right]
 \\
& \leq  2
\left(
\frac{A}{1+(|i-j|-L)^{1+\tau}}
\right)^2
\end{alignat*}
from \eqref{eq:lowe:bound:distij}. Hence, we can simply show
\begin{flalign*}
\left|
\frac{1}{m_n}
\sum_{i,j=1}^{m_n}
\mathbb{E}
\left[
\left(
K_{\theta}(\mu_i,\mu_j)
-
K_{\theta_0}(\mu_i,\mu_j)
\right)^2
\right]
-
T_{\theta}
\right|
\to^{a.s.}_{n \to \infty} 0,
\end{flalign*}
with
\[
T_{\theta}
=
 \sum_{j=-\infty}^{\infty}
\mathbb{E}
\left(
\left[
 K_{\theta}(\mu_0,\mu_j)
-
K_{\theta_0}(\mu_0,\mu_j)
 \right]^2
\right).
\]
From \eqref{eq:lowe:bound:distij} and Condition \ref{cond:asymptotics:six}, we can show that there exists a deterministic finite constant $C$ so that
\[
\sup_{\theta \in \Theta}
\max_{i=1,...,p}
\left|
\frac{\partial}{\partial \theta_i}
S_{\theta}
\right|
\leq C.
\]
Also, by dominated convergence $T_{\theta}$ is a continuously differentiable function of $\theta$ and
\[
\sup_{\theta \in \Theta}
\max_{i=1,...,p}
\left|
\frac{\partial}{\partial \theta_i}
T_{\theta}
\right|
\leq C'
\]
where $C'$ is also a deterministic finite constant.
Hence $\sup_{\theta \in \Theta} |S_{\theta} - T_{\theta}|\to 0$ almost surely as $n \to \infty$.

\end{proof}

\begin{proof}[Proof of Proposition \ref{prop:asymptotics:example}]

Assume that 
\[
\liminf \limits_{n\rightarrow \infty} \inf\limits_{\|\theta-\theta_0 \| \geq \alpha} \frac{1}{n} \sum_{i,j=1}^n \left[K_\theta(\mu_i,\mu_j)-K_{\theta_0}(\mu_i,\mu_j) \right]^2 = 0.
\]
Then from Proposition \ref{prop:for:proof:example} and by compacity, there exists $\theta_1 \neq \theta_0$ so that
\[
\sum_{j=-\infty}^{\infty}
\mathbb{E}
\left(
\left[
 K_{\theta_1}(\mu_0,\mu_j)
-
K_{\theta_0}(\mu_0,\mu_j)
 \right]^2
\right)
=0.
\]
From the conditions on $\{F_{\theta}\}$, there exist $\beta>0$, $\delta>0$, $a \geq 0$ so that for $u \in [a-\delta,a+\delta]$ we have $| F_{\theta_1}(u) - F_{\theta_0}(u) | \geq \beta$. Hence, we have
\[
\beta^2 P( W_2( \mu_0 , \mu_{k-1} ) \in [a- \delta , a+ \delta]  )
=0,
\]
for $k$ so that $ a \in (k-1,k]$. 

Let now $g_0: [0,L] \to \mathbb{R}^+$ be defined by $ g_0(u)= D_0 \exp(-1/(1-u^2)) \mathbf{1}_{\{ u \in [-1,1] \}}$ where $ 0 < D_0 < \infty$ is so that $\int_{\mathbb{R}} g_0(u) du = 1$. Then, $g_0$ is infinitely differentiable.
Let $h_0(u) = (1 / \sigma) g_0( (u-\delta/4) / \sigma )$ and $h_{k-1}(u) = (1 / \sigma) g_0( (u-a) / \sigma )$, where $\sigma >0$ is chosen small enough so that, with $\nu_0$ and $\nu_{k-1}$ the distributions with probability density functions $h_0$ and $h_{k-1}$ we have $W_2( \nu_0,\nu_{k-1} ) \in [ a- \delta/2 , a+ \delta/2 ]$ and $\nu_0,\nu_{k-1}$ have supports in $[0,L],[k-1,k-1+L]$.

Let now $P_1,P_2$ be two distributions with support in $[0,L]$, with quantile functions $q_1,q_2$, with cumulative distribution functions $F_1,F_2$ and with probability density functions $f_1,f_2$. Then we have
\begin{alignat*}{1} \label{eq:bound:wasserstein:with:pdf}
W_2( P_1,P_2 )  & = \sqrt{ \int_{0}^1 (q_1 - q_2)^2 } \\ 
& \leq \sqrt{L} \sqrt{ \int_{0}^1 |q_1 - q_2| } \\
&  = \sqrt{L} \sqrt{ \int_{0}^ L  |F_1 - F_2| } \\
& \leq L \sqrt{ \sup_{ u \in [0,L]}  |F_1(u) - F_2(u)| } \\
&  \leq L \sqrt{ \int_0^L  |f_1 - f_2| } \\
&  \leq L^{3/2} \sqrt{ \sup_{u \in [0,L]}  |f_1(u) - f_2 (u) | }. \notag
\end{alignat*}

Let $\tau>0$ be so that $ L^{3/2} \tau^{1/2} \leq \delta/5$.
Then, for any $f: [0,L] \to \mathbb{R}$ and $g:[k-1,k-1+L] \to \mathbb{R}$, we have that $|  f / ( \int_{0}^L f)  - h_0 |_{\infty} \leq \tau$ and $|  g / ( \int_{0}^L g) - h_{k-1} |_{\infty} \leq \tau$ imply $W_2( \nu_f , \nu_g ) \in [a- \delta , a+ \delta]$, where $\nu_f,\nu_g$ are the measures with probability density functions $f$ and $g$. 
Since $h_0$ and $h_{k-1}$ are infinitely differentiable, have integral one, and have respective supports included in $[0,L]$ and $[k-1,k-1+L]$, it is easy to see that there exists $\epsilon>0$ so that $| f - h_0 |_{\infty} \leq \epsilon$ implies $|  f / ( \int_{0}^L f)  - h_0 |_{\infty} \leq \tau$, and similarly for $g$ and $h_{k-1}$. Hence, if we can show that
\[
P( \sup_{u \in [0,L]} | h_0(u) - \exp(Z_0(u)) | \leq \epsilon )
>0
\] 
and
\[
P( \sup_{u \in [0,L]} | h_{k-1}(u + (k-1)) - \exp(Z_{k-1}(u)) | \leq \epsilon )
>0,
\] 
we obtain a contradiction. The two probabilities above are shown to be non-zero similarly and we will address the first one only. It is sufficient to show that
\[
P( \sup_{u \in [0,L]} | h_0(u) + \epsilon/2 - \exp(Z_0(u)) | \leq \epsilon/2 )
>0.
\] 
Since $h_0 + \epsilon/2$ is continuous and bounded away from $0$ and infinity on $[0,L]$, it is sufficient to show that for all $\kappa >0$,
\[
P( \sup_{u \in [0,L]} | \log(  h_0(u) + \epsilon / 2 ) - Z_0(u) | \leq \kappa )
>0.
\] 
From e.g. Theorem 1.1 in \cite{li1999approximation}, since $z_0$ has mean function zero, we have
\[
P( \sup_{u \in [0,L]} |   Z_0(u) | \leq \kappa )
>0.
\] 
Consider now the Gaussian measures $\mathcal{G}_1$ and $\mathcal{G}_2$, on the space of continuous functions from $[0,L] \to \mathbb{R}$, so that $\mathcal{G}_1$ is the measure of the Gaussian process $Z_0$ and $\mathcal{G}_2$ is that of $Z_0 - \log( h_0 + \epsilon/2) $.    
Then, from e.g. the discussion in (22) in Chapter 4.2 of \cite{stein99interpolation}, since $\log(h_0 + \epsilon / 2)$ is infinitely differentiable, and from the assumptions on the covariance function of $Z_0$, the Gaussian measures $\mathcal{G}_1$ and  $\mathcal{G}_2$ are equivalent. Hence, since 
\[
 \mathcal{G}_1
 \left( 
 \{
f \mbox{ continuous}: [0,L] \to \mathbb{R};
| f  |_{\infty} \leq \kappa
 \}
 \right)
 >0,
\]
we also have
\[
 \mathcal{G}_2
 \left( 
 \{
f \mbox{ continuous}: [0,L] \to \mathbb{R};
| f  |_{\infty} \leq \kappa
 \}
 \right)
 >0,
\]
which is exactly
\[
P( \sup_{u \in [0,L]} | \log(  h_0(u) + \epsilon / 2 ) - Z_0(u) | \leq \kappa )
>0.
\] 
This concludes the proof that Condition \ref{cond:asymptotics:cinq} holds.

The proof that Condition \ref{cond:asymptotics:huit} holds can be obtained in the same way. In particular, an analog of Proposition \ref{prop:for:proof:example} can be obtained.
We skip the details.
\end{proof}


\ifCLASSOPTIONcaptionsoff
  \newpage
\fi



%

\section*{Acknowledgements}
We thank the anonymous reviewers, whose suggestions have greatly contributed to improve the manuscript.

We thank Yann Richet, from the French Radioprotection and Nuclear Safety Institute (IRSN), for introducing us to the problem of
axial burn up analysis of fuel pins, which motivated the present work.

\bibliographystyle{IEEEtran}
\bibliography{biblio-glm-IEEE}

\begin{thebibliography}{10}
\providecommand{\url}[1]{#1}
\csname url@samestyle\endcsname
\providecommand{\newblock}{\relax}
\providecommand{\bibinfo}[2]{#2}
\providecommand{\BIBentrySTDinterwordspacing}{\spaceskip=0pt\relax}
\providecommand{\BIBentryALTinterwordstretchfactor}{4}
\providecommand{\BIBentryALTinterwordspacing}{\spaceskip=\fontdimen2\font plus
\BIBentryALTinterwordstretchfactor\fontdimen3\font minus
  \fontdimen4\font\relax}
\providecommand{\BIBforeignlanguage}[2]{{%
\expandafter\ifx\csname l@#1\endcsname\relax
\typeout{** WARNING: IEEEtran.bst: No hyphenation pattern has been}%
\typeout{** loaded for the language `#1'. Using the pattern for}%
\typeout{** the default language instead.}%
\else
\language=\csname l@#1\endcsname
\fi
#2}}
\providecommand{\BIBdecl}{\relax}
\BIBdecl

\bibitem{SFDS}
N.~Venet, F.~Bachoc, F.~Gamboa, and J.-M. Loubes, ``Modèles de régression
  gaussienne pour des distributions en entrée,'' \emph{49è Journées de
  statistique}, 2016.

\bibitem{MR1127423}
N.~A.~C. Cressie, \emph{Statistics for spatial data}, ser. Wiley Series in
  Probability and Mathematical Statistics: Applied Probability and
  Statistics.\hskip 1em plus 0.5em minus 0.4em\relax John Wiley \& Sons, Inc.,
  New York, 1991, a Wiley-Interscience Publication.

\bibitem{MR2514435}
C.~E. Rasmussen and C.~K.~I. Williams, \emph{Gaussian processes for machine
  learning}, ser. Adaptive Computation and Machine Learning.\hskip 1em plus
  0.5em minus 0.4em\relax MIT Press, Cambridge, MA, 2006.

\bibitem{vapnik1997support}
V.~Vapnik, S.~E. Golowich, A.~Smola \emph{et~al.}, ``Support vector method for
  function approximation, regression estimation, and signal processing,''
  \emph{Advances in neural information processing systems}, pp. 281--287, 1997.

\bibitem{scholkopf2002learning}
B.~Sch{\"o}lkopf and A.~J. Smola, \emph{Learning with kernels: support vector
  machines, regularization, optimization, and beyond}.\hskip 1em plus 0.5em
  minus 0.4em\relax MIT press, 2002.

\bibitem{cristianini2000support}
N.~Cristianini and J.~Shawe-Taylor, ``Support vector machines,'' 2000.

\bibitem{cohenlifshits}
\BIBentryALTinterwordspacing
S.~Cohen and M.~A. Lifshits, ``{Stationary {G}aussian random fields on
  hyperbolic spaces and on {E}uclidean spheres},'' \emph{ESAIM Probab. Stat.},
  vol.~16, pp. 165--221, 2012. [Online]. Available:
  \url{http://dx.doi.org/10.1051/ps/2011105}
\BIBentrySTDinterwordspacing

\bibitem{istas2011manifold}
\BIBentryALTinterwordspacing
J.~Istas, ``{Manifold indexed fractional fields},'' \emph{ESAIM Probab. Stat.},
  vol.~16, pp. 222--276, 2012. [Online]. Available:
  \url{http://dx.doi.org/10.1051/ps/2011106}
\BIBentrySTDinterwordspacing

\bibitem{feragen2015geodesic}
A.~Feragen, F.~Lauze, and S.~Hauberg, ``Geodesic exponential kernels: When
  curvature and linearity conflict,'' in \emph{Proceedings of the IEEE
  Conference on Computer Vision and Pattern Recognition}, 2015, pp. 3032--3042.

\bibitem{flaxman2015supported}
S.~R. Flaxman, Y.-X. Wang, and A.~J. Smola, ``Who supported obama in 2012?:
  Ecological inference through distribution regression,'' in \emph{Proceedings
  of the 21th ACM SIGKDD International Conference on Knowledge Discovery and
  Data Mining}.\hskip 1em plus 0.5em minus 0.4em\relax ACM, 2015, pp. 289--298.

\bibitem{lopez2015towards}
D.~Lopez-Paz, K.~Muandet, B.~Sch{\"o}lkopf, and I.~Tolstikhin, ``Towards a
  learning theory of cause-effect inference,'' in \emph{International
  Conference on Machine Learning}, 2015, pp. 1452--1461.

\bibitem{muandet2017kernel}
K.~Muandet, K.~Fukumizu, B.~Sriperumbudur, B.~Sch{\"o}lkopf \emph{et~al.},
  ``Kernel mean embedding of distributions: A review and beyond,''
  \emph{Foundations and Trends{\textregistered} in Machine Learning}, vol.~10,
  no. 1-2, pp. 1--141, 2017.

\bibitem{poczos2012nonparametric}
B.~P{\'o}czos, L.~Xiong, and J.~Schneider, ``Nonparametric divergence
  estimation with applications to machine learning on distributions,''
  \emph{arXiv preprint arXiv:1202.3758}, 2012.

\bibitem{poczos2013distribution}
B.~P{\'o}czos, A.~Singh, A.~Rinaldo, and L.~A. Wasserman, ``Distribution-free
  distribution regression.'' in \emph{AISTATS}, 2013, pp. 507--515.

\bibitem{barnabas2012nonparametric}
D.~J. S. J.~S. Barnabas and L.~X. Poczos, ``Nonparametric kernel estimators for
  image classification,'' in \emph{CVPR}, vol. 2012, 2012, p.~1.

\bibitem{szabo2015two}
Z.~Szab{\'o}, A.~Gretton, B.~P{\'o}czos, and B.~Sriperumbudur, ``Two-stage
  sampled learning theory on distributions,'' in \emph{Artificial Intelligence
  and Statistics}, 2015, pp. 948--957.

\bibitem{2016arXiv160509522M}
K.~{Muandet}, K.~{Fukumizu}, B.~{Sriperumbudur}, and B.~{Sch{\"o}lkopf},
  ``{Kernel Mean Embedding of Distributions: A Review and Beyonds},''
  \emph{ArXiv e-prints}, May 2016.

\bibitem{KolouriZouRohde}
S.~Kolouri, Y.~Zou, and G.~K. Rohde, ``Sliced wasserstein kernels for
  probability distributions,'' in \emph{Proceedings of the IEEE Conference on
  Computer Vision and Pattern Recognition}, 2016, pp. 5258--5267.

\bibitem{villani2009optimal}
C.~Villani, \emph{Optimal transport: old and new}.\hskip 1em plus 0.5em minus
  0.4em\relax Springer Science \& Business Media, 2009, vol. 338.

\bibitem{MR1625620}
\BIBentryALTinterwordspacing
A.~Munk and C.~Czado, ``{Nonparametric validation of similar distributions and
  assessment of goodness of fit},'' \emph{J. R. Stat. Soc. Ser. B Stat.
  Methodol.}, vol.~60, no.~1, pp. 223--241, 1998. [Online]. Available:
  \url{http://dx.doi.org/10.1111/1467-9868.00121}
\BIBentrySTDinterwordspacing

\bibitem{MR3338645}
\BIBentryALTinterwordspacing
E.~Boissard, T.~Le~Gouic, and J.-M. Loubes, ``Distribution's template estimate
  with {W}asserstein metrics,'' \emph{Bernoulli}, vol.~21, no.~2, pp. 740--759,
  2015. [Online]. Available: \url{http://dx.doi.org/10.3150/13-BEJ585}
\BIBentrySTDinterwordspacing

\bibitem{Gouic2016}
\BIBentryALTinterwordspacing
T.~Le~Gouic and J.-M. Loubes, ``Existence and consistency of {W}asserstein
  barycenters,'' \emph{Probability Theory and Related Fields}, pp. 1--17, 2016.
  [Online]. Available: \url{http://dx.doi.org/10.1007/s00440-016-0727-z}
\BIBentrySTDinterwordspacing

\bibitem{peyre2016gromov}
G.~Peyr{\'e}, M.~Cuturi, and J.~Solomon, ``Gromov-{W}asserstein averaging of
  kernel and distance matrices,'' in \emph{ICML 2016}, 2016.

\bibitem{poczos13distribution}
B.~P\'oczos, A.~Singh, A.~Rinaldo, and L.~Wasserman, ``Distribution-free
  distribution regression,'' in \emph{In Proceedings of the 16th International
  Conference on Artificial Intelligence and Statistics, volume 31 of JMLR
  Proceedings}, 2013, pp. 507--515.

\bibitem{radulescu2009sensitivity}
G.~Radulescu, D.~E. Mueller, and J.~C. Wagner, ``Sensitivity and uncertainty
  analysis of commercial reactor criticals for burnup credit,'' \emph{Nuclear
  Technology}, vol. 167, no.~2, pp. 268--287, 2009.

\bibitem{cacciapouti2000axial}
R.~Cacciapouti, ``Axial burnup profile database for pressurized water
  reactors.'' Oak Ridge National Laboratory (ORNL), Oak Ridge, TN (United
  States), Tech. Rep., 2000.

\bibitem{bowman2003scale}
S.~M. Bowman, D.~F. Hollenbach, M.~D. DeHART, B.~T. Rearden, I.~C. Gauld, and
  S.~Goluoglu, ``Scale 5: Powerful new criticality safety analysis tools,''
  \emph{Nuclear Science and Technology}, 2003.

\bibitem{rachev}
S.~T. Rachev, ``Monge-kantorovich problem on mass transfer and its applications
  in stochastics,'' \emph{Teoriya Veroyatnostei i ee Primeneniya}, vol.~29,
  no.~4, pp. 625--653, 1984.

\bibitem{lifshits2012lectures}
M.~Lifshits, ``Lectures on gaussian processes,'' in \emph{Lectures on Gaussian
  Processes}.\hskip 1em plus 0.5em minus 0.4em\relax Springer, 2012, pp.
  1--117.

\bibitem{kloeckner2010geometric}
B.~Kloeckner, ``A geometric study of wasserstein spaces: Euclidean spaces,''
  \emph{Annali della Scuola Normale Superiore di Pisa-Classe di Scienze-Serie
  V}, vol.~9, no.~2, p. 297, 2010.

\bibitem{mandelbrot1968fractional}
B.~B. Mandelbrot and J.~W. Van~Ness, ``Fractional brownian motions, fractional
  noises and applications,'' \emph{SIAM review}, vol.~10, no.~4, pp. 422--437,
  1968.

\bibitem{berg_al}
C.~Berg, J.~P.~R. Christensen, and P.~Ressel, \emph{Harmonic analysis on
  semigroups}.\hskip 1em plus 0.5em minus 0.4em\relax Springer-Verlag, 1984.

\bibitem{stein99interpolation}
M.~Stein, \emph{Interpolation of Spatial Data: Some Theory for
  {Kriging}}.\hskip 1em plus 0.5em minus 0.4em\relax Springer, New York, 1999.

\bibitem{bachoc13cross}
F.~Bachoc, ``Cross validation and maximum likelihood estimations of
  hyper-parameters of {Gaussian} processes with model mispecification,''
  \emph{Computational Statistics and Data Analysis}, vol.~66, pp. 55--69, 2013.

\bibitem{bachoc16asymptotic}
------, ``Asymptotic analysis of covariance parameter estimation for gaussian
  processes in the misspecified case,'' \emph{Bernoulli, forthcoming}, 2016.

\bibitem{zhang10kriging}
H.~Zhang and Y.~Wang, ``{Kriging} and cross validation for massive spatial
  data,'' \emph{Environmetrics}, vol.~21, pp. 290--304, 2010.

\bibitem{MarMar1984}
K.~Mardia and R.~Marshall, ``Maximum likelihood estimation of models for
  residual covariance in spatial regression,'' \emph{Biometrika}, vol.~71, pp.
  135--146, 1984.

\bibitem{cressie93asymptotic}
N.~Cressie and S.~Lahiri, ``The asymptotic distribution of {REML} estimators,''
  \emph{Journal of Multivariate Analysis}, vol.~45, pp. 217--233, 1993.

\bibitem{cressie96asymptotics}
------, ``Asymptotics for {REML} estimation of spatial covariance parameters,''
  \emph{Journal of Statistical Planning and Inference}, vol.~50, pp. 327--341,
  1996.

\bibitem{shaby12tapered}
B.~A. Shaby and D.~Ruppert, ``Tapered covariance: Bayesian estimation and
  asymptotics,'' \emph{Journal of Computational and Graphical Statistics},
  vol.~21, no.~2, pp. 433--452, 2012.

\bibitem{bachoc14asymptotic}
F.~Bachoc, ``Asymptotic analysis of the role of spatial sampling for covariance
  parameter estimation of {G}aussian processes,'' \emph{Journal of Multivariate
  Analysis}, vol. 125, pp. 1--35, 2014.

\bibitem{furrer16asymptotic}
R.~Furrer, F.~Bachoc, and J.~Du, ``Asymptotic properties of multivariate
  tapering for estimation and prediction,'' \emph{Journal of Multivariate
  Analysis}, vol. 149, pp. 177--191, 2016.

\bibitem{zhang04inconsistent}
H.~Zhang, ``Inconsistent estimation and asymptotically equivalent
  interpolations in model-based geostatistics,'' \emph{Journal of the American
  Statistical Association}, vol.~99, pp. 250--261, 2004.

\bibitem{AEPRFMCF}
M.~Stein, ``Asymptotically efficient prediction of a random field with a
  misspecified covariance function,'' \emph{The Annals of Statistics}, vol.~16,
  pp. 55--63, 1988.

\bibitem{BELPUICF}
------, ``Bounds on the efficiency of linear predictions using an incorrect
  covariance function,'' \emph{The Annals of Statistics}, vol.~18, pp.
  1116--1138, 1990.

\bibitem{UAOLPRFUISOS}
------, ``Uniform asymptotic optimality of linear predictions of a random field
  using an incorrect second-order structure,'' \emph{The Annals of Statistics},
  vol.~18, pp. 850--872, 1990.

\bibitem{putter01effect}
H.~Putter and G.~A. Young, ``On the effect of covariance function estimation on
  the accuracy of {K}riging predictors,'' \emph{Bernoulli}, vol.~7, no.~3, pp.
  421--438, 2001.

\bibitem{Yin1991}
Z.~Ying, ``Asymptotic properties of a maximum likelihood estimator with data
  from a {G}aussian process,'' \emph{Journal of Multivariate Analysis},
  vol.~36, pp. 280--296, 1991.

\bibitem{Yin1993}
------, ``Maximum likelihood estimation of parameters under a spatial sampling
  scheme,'' \emph{The Annals of Statistics}, vol.~21, pp. 1567--1590, 1993.

\bibitem{CheSimYin2000}
H.-S. Chen, D.~Simpson, and Z.~Ying, ``Infill asymptotics for a stochastic
  process model with measurement error,'' \emph{Statistica Sinica}, vol.~10,
  pp. 141--156, 2000.

\bibitem{ESCMSGRFM}
W.~Loh and T.~Lam, ``Estimating structured correlation matrices in smooth
  {Gaussian} random field models,'' \emph{The Annals of Statistics}, vol.~28,
  pp. 880--904, 2000.

\bibitem{Loh2005}
W.-L. Loh, ``Fixed-domain asymptotics for a subclass of {M}at{\'e}rn-type
  {G}aussian random fields,'' \emph{The Annals of Statistics}, vol.~33, pp.
  2344--2394, 2005.

\bibitem{bachoc2016smallest}
F.~Bachoc and R.~Furrer, ``On the smallest eigenvalues of covariance matrices
  of multivariate spatial processes,'' \emph{Stat}, 2016.

\bibitem{roustant12dicekriging}
O.~Roustant, D.~Ginsbourger, and Y.~Deville, ``{DiceKriging}, {DiceOptim}: Two
  {R} packages for the analysis of computer experiments by kriging-based
  metamodelling and optimization,'' \emph{Journal of Statistical Software},
  vol.~51, no.~1, pp. 1--55, 2012.

\bibitem{Muehlenstaedt2016}
\BIBentryALTinterwordspacing
T.~Muehlenstaedt, J.~Fruth, and O.~Roustant, ``Computer experiments with
  functional inputs and scalar outputs by a norm-based approach,''
  \emph{Statistics and Computing}, pp. 1--15, 2016. [Online]. Available:
  \url{http://dx.doi.org/10.1007/s11222-016-9672-z}
\BIBentrySTDinterwordspacing

\bibitem{nanty2016sampling}
S.~Nanty, C.~Helbert, A.~Marrel, N.~P{\'e}rot, and C.~Prieur, ``Sampling,
  metamodeling, and sensitivity analysis of numerical simulators with
  functional stochastic inputs,'' \emph{SIAM/ASA Journal on Uncertainty
  Quantification}, vol.~4, no.~1, pp. 636--659, 2016.

\bibitem{ramsay05functional}
J.~O. Ramsay and B.~W. Silverman, \emph{Functional Data Analysis}.\hskip 1em
  plus 0.5em minus 0.4em\relax New York: Springer, 2005, vol. 338.

\bibitem{furrer2006covariance}
R.~Furrer, M.~G. Genton, and D.~Nychka, ``Covariance tapering for interpolation
  of large spatial datasets,'' \emph{Journal of Computational and Graphical
  Statistics}, vol.~15, no.~3, pp. 502--523, 2006.

\bibitem{Venet_critical}
\BIBentryALTinterwordspacing
N.~Venet, ``On the existence of fractional brownian fields indexed by manifolds
  with closed geodesics,'' \emph{arXiv preprint}, 2016. [Online]. Available:
  \url{https://arxiv.org/abs/1612.05984}
\BIBentrySTDinterwordspacing

\bibitem{adams03sobolev}
R.~Adams and J.~Fournier, \emph{Sobolev spaces}.\hskip 1em plus 0.5em minus
  0.4em\relax Academic Press, Amsterdam, 2003.

\bibitem{hu97strong}
\BIBentryALTinterwordspacing
T.-C. Hu and R.~Taylor, ``\BIBforeignlanguage{eng}{On the strong law for arrays
  and for the bootstrap mean and variance.}''
  \emph{\BIBforeignlanguage{eng}{International Journal of Mathematics and
  Mathematical Sciences}}, vol.~20, no.~2, pp. 375--382, 1997. [Online].
  Available: \url{http://eudml.org/doc/47832}
\BIBentrySTDinterwordspacing

\bibitem{li1999approximation}
W.~V. Li, W.~Linde \emph{et~al.}, ``Approximation, metric entropy and small
  ball estimates for gaussian measures,'' \emph{The Annals of Probability},
  vol.~27, no.~3, pp. 1556--1578, 1999.

\end{thebibliography}

%








\end{document}